\documentclass[12pt,reqno]{imsart}

\usepackage{setup_normal}
\usepackage{newcommands}
\usepackage{float}
\usepackage{constants}
\usepackage{titletoc}
\usepackage{stackrel}
\usepackage{float}
\usepackage{enumerate}
\usepackage[shortlabels]{enumitem}
\usepackage{vwcol}
\usepackage{todonotes}
\newconstantfamily{c}{symbol=c}

\usepackage[indent=2em, skip=0.49em]{parskip}

\usepackage{amsmath}
\usepackage{amssymb}
\usepackage{amsthm}
\usepackage{amsfonts}
\usepackage{mathrsfs}

\usepackage{dsfont}
\usepackage{microtype}

\usepackage{constants}

\usepackage{bbm}

\usepackage{stmaryrd}
\newcommand{\bkt}[1]{\llbracket #1\rrbracket}

\usepackage[numbers]{natbib}

\usepackage[ruled]{algorithm}
\usepackage{algpseudocode}

\usepackage{booktabs}
\usepackage{multirow}

\usepackage{float}

\newcommand{\proj}{\Pi}
\newcommand{\laser}{\textsf{LASER}}

\usepackage{mathabx}
\let\bar\widebar
\let\tilde\widetilde
\let\hat\widehat

\begin{document}

\begin{frontmatter}
    \title{A new approach to locally adaptive polynomial regression}
    \runtitle{Locally adaptive estimation}  
    
    \begin{aug}
        \author[A]{\fnms{Sabyasachi} \snm{Chatterjee}\ead[label=e1]{sc1706@illinois.edu}},
        \author[B]{\fnms{Subhajit} \snm{Goswami}\ead[label=e2]{goswami@math.tifr.res.in}}, 
        \author[C]{\fnms{Soumendu Sundar} \snm{Mukherjee}\ead[label=e3]{ssmukherjee@isical.ac.in}}
        \footnote{Author names are sorted alphabetically.}
        
        \runauthor{Chatterjee, Goswami and Mukherjee}
        
        \affiliation{University of Illinois Urbana-Champaign, Tata Institute of Fundamental Research and Indian Statistical Institute
        }

        \address[A]{Department of Statistics\\ 
            University of Illinois Urbana–Champaign\\
            \href{mailto:sc1706@illinois.edu}{sc1706@illinois.edu}\\
            \phantom{as}}
        \address[B]{School of Mathematics\\
            Tata Institute of Fundamental Research\\
            \href{mailto:ssmukherjee@isical.ac.in}{goswami@math.tifr.res.in}\\
            \phantom{ac}}
\address[C]{Statistics and Mathematics Unit (SMU)\\
            Indian Statistical Institute, Kolkata\\
            \href{mailto:goswami@math.tifr.res.in}{ssmukherjee@isical.ac.in}\\
            \phantom{ad}}
        
    \end{aug}

    \begin{abstract}
Adaptive bandwidth selection is a fundamental challenge in nonparametric 
regression. %
This paper introduces a new bandwidth selection procedure inspired 
by the optimality criteria for $\ell_0$-penalized regression. Although similar 
in spirit to Lepski's method and its variants in selecting the 
largest interval satisfying an admissibility 
criterion, our approach stems from a distinct philosophy, utilizing criteria 
based on $\ell_2$-norms of interval projections rather than explicit point and 
variance estimates. We obtain non-asymptotic risk bounds for the local polynomial regression methods based on our bandwidth selection procedure 
which adapt (near-)optimally to the local H\"{o}lder exponent of the underlying regression function simultaneously at all points in 
its domain. Furthermore, we show that there is a single ideal choice of a global tuning parameter in each case under which the above-mentioned 
local adaptivity holds. %
The optimal risks of our methods derive from the properties of %
solutions to a %
new ``bandwidth selection equation'' %
which is of independent interest. %
We believe that the principles underlying our approach provide a new 
perspective to the classical yet ever relevant problem of locally adaptive 
nonparametric regression.

    \end{abstract}
    
\end{frontmatter}

\section{Introduction}\label{sec:intro}
\subsection{Nonparametric regression: local adaptivity}\label{subsec:nonparam}
Nonparametric regression is a classical and fundamental problem in Statistics; see~\cite{Gyorfi, WassermanNonpar, Tsybakovbook} \sloppy for an introduction to the subject. The basic problem is to estimate the conditional expectation function $f(x) = \E (Y|X = x)$ from data points $\{x_i,y_i\}_{i = 1}^{n}$ %
under weak assumptions on $f$, such as $f$ belongs to some infinite dimensional function class like all Lipschitz/H\"{o}lder smooth functions. In this paper, we develop a new locally adaptive nonparametric regression method. To keep our exposition simple and focused, we only consider the univariate case. Our estimator and the theoretical analysis of its performance can both be extended to a multivariate setting. However, this requires several new ingredients and will be carried out in a forthcoming work.

Let us consider the simplest possible setting where the design points $\{x_i\}_{i = 1}^{n}$ are fixed to be on a grid in $[0,1]$, i.e., $$x_1 < x_2 < \dots < x_n,$$ where $x_i = \frac{i}{n}.$ In this case, denoting $\theta^*_i = f(\frac{i}{n})$ we have the usual signal plus noise model
\begin{equation}\label{eq:thetaep}
    y_i = \theta^*_i + \epsilon_i.
\end{equation}
We also make the standard assumption that $\epsilon_i$'s are independent mean zero sub-Gaussian variables with {\em sub-Gaussian norm} bounded by $\sigma > 
0$, i.e., 
\begin{equation}\label{eq:psi2norm}
\sup_{i \in [n]}\E\left[\exp(\tfrac{\epsilon_i^2}{\sigma^2})\right] \le 1. 
\end{equation}
Under this standard model, the task is to estimate the unknown function/signal $f$/$\theta^*$ upon observing the data vector $y$.

In many real problems, the true regression function $f$ may not be uniformly 
smooth, and its degree of smoothness may vary in different parts of the 
domain. It should be easier to estimate a function where it is smooth and 
harder where it is rough. In this article, we revisit the phenomena of \textit{local 
adaptivity}. Intuitively, we can say that a nonparametric regression method is locally adaptive if it estimates the function at each location in the domain ``as good as possible'' depending on the local degree of smoothness.

Although the collection of methods in the nonparametric regression toolbox is 
very rich (see, e.g.,~\cite{cover1967nearest, fan1996framework,wand1994kernel,de1978practical,green1993nonparametric,wahba1990spline, smola1998learning, donoho1994ideal, tibshirani2020divided, breiman2017classification, bishop1994neural} and the references 
therein), not all of them are provably locally adaptive in the sense alluded to 
in the last paragraph. In fact, it is well known (see, e.g.,~\cite{donoho1998minimax}) that linear smoothers, like local polynomial 
regression, kernel smoothing, smoothing splines, etc., are {\em not} locally 
adaptive.

A large class of nonlinear methods which are regarded as locally adaptive use 
kernel smoothing/polynomial regression with data-dependent variable 
bandwidths. One stream of such methods originates from the seminal work of 
Lepski~\cite{lepskii1991problem, lepskii1992asymptotically}. In a nutshell, at each point $x$ in the 
domain, Lepski's method chooses the largest bandwidth (from a discrete set of 
possible values) such that the kernel estimate at $x$ is within a carefully 
defined error tolerance to estimates at smaller bandwidths. For a state-of-the-art account of the developments in this area, see the ICM survey 
\cite{lepski2022theory} by Lepski and the references therein. There is also a large body of work on variable bandwidth local polynomial 
regression; some notable works include~\cite{FanGijbels92, 
FanGijbels95, fan1996local, RupertWand94, 
RuppertSheathWand95, GN97, lafferty2008rodeo}. %

On the other hand, certain global methods --- notably the ones utilizing 
penalized least squares --- are also known or widely believed to be locally 
adaptive. Locally adaptive regression splines~\cite{koenker1994quantile, mammen1997locally}, trend filtering \cite{steidl2006splines, kim2009ell_1, tibshirani2014adaptive}, wavelet thresholding~\cite{mallat1999wavelet, johnstone2011gaussian, Cai99, CaiZhou2009}, 
jump/$\ell_0$-penalized least squares~\cite{boysen2009consistencies}, dyadic CART~\cite{donoho1997cart, chatterjee2019adaptive} all fall under this approach. 

The current work presents a fusion between the above two approaches in 
that we develop a new criterion for (variable) bandwidth selection inspired 
by $\ell_0$-penalized least squares. In the next subsection, we provide a brief 
sketch illustrating how this criterion is developed.

\subsection{From $\ell_0$-penalized regression to bandwidth 
selection}\label{subsec:informal}

In the context of variable bandwidth estimators, one prominent and classical 
approach for selecting the optimal bandwidth is based on an explicit 
optimization of the bias-variance decomposition of the estimation error (see, 
for instance, \cite{FanGijbels92, FanGijbels95, GN97}). As hinted in the 
previous subsection, we develop a new bandwidth selection procedure 
inspired by the optimality criterion in the $\ell_0$-penalized regression. %
At the heart of our approach lies a new {\em discrepancy 
measure}, the formulation of which we believe to be one of the distinguishing contributions of this article. 

We now present a sketch of how we arrive at our discrepancy measure 
starting from the optimality criterion of the $\ell_0$-penalized least squares 
problem:
\begin{equation}\label{eq:argmin}
    \argmin_{\theta \in \R^n} \big(\|y - \theta\|^2 + \lambda 
    \|D^{(r)}\theta\|_0\big)
\end{equation}
where $D^{(r)}$ is the $r$-th order finite difference operator. For simplicity of 
discussion, let us confine ourselves to the case $r = 1$, where the optimal 
solution $\hat \theta$ is given by the average value of $y$ over an optimal partition $\mathcal P$ of 
$\{1, \ldots, n\}$. If we sub-divide any block $I \in \mathcal P$ into two sub-blocks $I_1$ and $I_2$ 
(say), then by the optimality of $\hat \theta$ we can write,
\begin{equation*}
\sum_{i \in I} (y_i - \overline y_I)^2 \le \sum_{i \in I_1} (y_i - \overline y_{I_1})^2 + \sum_{i \in I_2} (y_i - \overline y_{I_2})^2 + \lambda.
\end{equation*}
Since this inequality holds for all sub-divisions of $I$ into two sub-blocks, we have
\begin{equation*}
    T_y^2(I) \stackrel{{\rm def.}}{=} \max_{I_1, I_2} \big( \sum_{i \in I} (y_i - \overline y_I)^2 - \sum_{i \in I_1} (y_i - \overline y_{I_1})^2 + \sum_{i \in I_2} (y_i - \overline y_{I_2})^2  \big) \le \lambda.
\end{equation*}
Thus $T_y(I)$ emerges as a goodness-of-fit measure (referred to as {\em local discrepancy measure} in the paper) for fitting a constant 
function to $y$ on the interval $I$. We would like to point out that the criterion 
``$T_y^2(I) \le \lambda$'' is closely related to the standard splitting criterion for 
regression in Classification and Regression Trees (CART). See \S\ref{subsec:reg_tree} 
below for more on this connection.

Going back to our regression problem \eqref{eq:thetaep}, we now propose to 
estimate $\theta_{i_0}^\ast$ for any given $i_0 \in \{1, \ldots, n\}$ as
\begin{equation*}
    \hat \theta_{i_0} = \overline y_I, \text{ where } I = {\arg \max}_{T_y^2(J) \le \lambda, \, J \ni i_0} \, |J|.
\end{equation*}
In words, we estimate $\theta_{i_0}^\ast$ by the average of $y$ over the {\em 
largest} interval $I$ containing $i_0$ such that $T_y^2(I) \le \lambda$. To understand why 
we take the largest interval, consider the {\em noiseless} scenario where the 
underlying signal is piecewise constant on some partition $\mathcal P^\ast$ of $\{1, \ldots, 
n\}$. In this case, we have $T_y^2(I) 
= 0$ for every sub-interval $I$ of the block $I_0  \in \mathcal P^\ast$ 
containing $i_0$. Clearly, the optimal bandwidth is given by $I_0$ which, by the 
previous observation, is also the largest interval $I$ containing $i_0$ such that 
$T_y^2(I) \le \lambda = 0$. In the general noisy setting, we would have to set $\lambda > 0$; however, it 
is not a priori clear why the prescription of the largest interval $I$ satisfying 
$T_y^2(I) \le \lambda$ would still yield a ``good'' bandwidth. This is precisely what we 
establish by setting $\lambda$ as  the ``effective noise" level of our problem (see the next subsection).  

More generally, for any $r \ge 1$, the relevant notion of goodness-of-fit turns out to be 
\begin{equation}\label{def:Tr}
			\begin{split}
				(T_y^{(r-1)}(I))^2 &= \max_{I_1,I_2} \|y_{I} -  \Pi_I^{(r - 1)} \theta_{I}\|^2 -  \|y_{I_1} - \Pi_{I_1}^{(r - 1)} y_{I_1}\|^2 -  \|y_{I_2} -  \Pi_{I_2}^{(r - 1)} y_{I_2}\|^2,
			\end{split}\end{equation} where $\Pi_I^{(r-1)}$ denotes the projection operator onto the subspace of $(r - 1)$-th degree polynomials and the corresponding estimator becomes
            \begin{equation*}
                \hat \theta_{i_0} = \Pi_I^{(r-1)}y_I \mbox{ where } I = {\arg \max}_{(T_y^{(r-1)}(J))^2 \le \lambda, J \ni i_0} \, |J|.
            \end{equation*}
We show in this work that the %
idea of choosing the largest interval containing any given $i_0 \in \{1, \ldots, n\}$ satisfying the 
criterion $(T_y^{(r-1)}(I))^2 \leq \lambda$ yields a good bandwidth for all degrees $r \geq 1$ and suitable families of signals $\theta^\ast$ 
using a unified argument. %

At a high level, our bandwidth selection method is similar in spirit to Lepski's 
method and its many incarnations in the sense that we choose our final 
bandwidth to be the largest interval satisfying a certain ``admissibility'' criterion. However, we would like to emphasize that the philosophy which 
leads us to our proposed criterion is markedly different from the one 
underlying Lepski's and related methods. In the next subsection, we elaborate 
these connections/differences further and give a summary of our main 
contributions.

\subsection{Our contributions vis-\`{a}-vis related works}\label{subsec:contribution}

There are only a handful of nonparametric regression methods which are \emph{provably} adaptive to the local H\"{o}lder smoothness exponent at each point in the domain with a single global choice of the tuning parameter value as well as being efficiently implementable. Our proposed method becomes a new 
member of this sparse toolbox (see our Theorem~\ref{thm:main} below).  

The first such \emph{provably} locally adaptive method was perhaps Lepski's 
method developed in \cite{lepskii1991problem, lepskii1992asymptotically} and by now 
there are several variants. Among these methods, perhaps the ``closest'' to the 
current work 
is that of Goldenshluger and Nemirovski \cite{GN97} (see also 
\cite{nemirovski2000topics}). Like our procedure, Goldenshluger and Nemirovski also 
perform local polynomial regression over a interval around any given 
location $i_0$, which is chosen as the maximal (symmetric) interval satisfying a 
certain ``goodness'' condition. We now briefly describe the motivation behind their 
notion of goodness. Let $h_0^*$ denote the ideal local bandwidth at 
$i_0$ obtained from the theoretical bias-variance trade-off of a local polynomial 
regression fit. With high probability, for any symmetric interval $I$ around $i_0$ 
of half-width at most $h_0^*$, the standard confidence set for $\theta_{i_0}^*$ (constructed from a fit over $I$) contains $\theta_{i_0}^*$. As such, the confidence 
sets contructed from fits over (symmetric) sub-intervals of $I$ have a non-empty intersection. 
Goldenshluger and Nemirovski call such intervals $I$ as ``good''. Naturally, the 
half-width of the largest good interval gives an estimate of the ideal bandwidth 
$h_0^*$.

In contrast, the bandwidth selection procedure we propose here %
is motivated by the optimality criterion in the $\ell_0$-penalized least squares problem \eqref{eq:argmin}.  Our ``goodness'' 
criterion is based on the discrepancy measure \eqref{def:Tr} which is formulated 
in terms of (squared) $\ell_2$-norms of ``interval projections'' at various 
scales/bandwidths. Although not at all obvious from its formulation, %
this recipe still attains the optimal rates 
owing to a key property satisfied by our selected interval $I$, namely,
\begin{equation}\label{eq:bse}
 T_{\theta^\ast}^{(r)}(I) \asymp \sigma \sqrt{\log n}
\end{equation}
which we refer to as the ``bandwidth selection equation''. See \S\ref{subsec:intuition} below for a detailed discussion on this at a heuristic level. A 
notable feature of \eqref{eq:bse} is that we do not need to adjust the contribution of the noise separately for different intervals $I$ unlike 
the class of methods discussed in the previous paragraph. %
It turns out that the intervals satisfying \eqref{eq:bse} automatically provide suitable control 
on the %
noise. Our procedure thus points to a {\em new} 
way %
for ensuring local adaptivity of variable bandwidth estimators. Also, one 
fortuitous consequence of basing our approach on penalized least squares is 
that we are able to leverage elementary properties of %
polynomials, %
thereby considerably simplifying our proofs.

As discussed towards the end of Section~\ref{subsec:nonparam}, some 
penalized least squares methods such as trend filtering, wavelet thresholding, jump penalized least squares, dyadic 
CART, etc. are also considered to be locally adaptive. Of these, to the best of 
our knowledge, only a particular variant of wavelet thresholding is provably 
locally adaptive~\cite{Cai99} in the sense we consider in this article. For instance, the existing MSE bounds~\cite{tibshirani2014adaptive,guntuboyina2020adaptive, ortelli2021prediction,zhang2023element} for trend filtering %
suggest that one needs to set 
the tuning parameter $\lambda$ differently in order to attain the optimal rate of 
convergence for different (smoothness) classes of functions.

Coming back to the (univariate) regression tree methods, such as dyadic CART or 
$\ell_0$-penalized least squares, our proof technique and insights suggest that 
these tree-based methods could be locally adaptive with a single ideal choice 
of the tuning parameter. We plan to investigate this in a future work.

\vspace{0.3cm}

We now summarize the main contributions of this article. 

\begin{itemize}
    \item We develop a new discrepancy measure and a criterion based on it for performing bandwidth selection inspired by the splitting criterion used in regression 
trees, thereby connecting tree based methods with variable bandwidth 
nonparametric regression.

    \item We show that our proposed estimator adapts to the local H\"{o}lder exponent and the local H\"{o}lder coefficient of the true 
    regression function.

    \item Our proof reveals a new way by which a variable bandwidth estimator can exhibit local adaptivity. %

    \item Our estimator has only one global tuning parameter $\lambda$ which, when set to $C \sigma \sqrt{\log n}$ for a small constant $C$, selects near-optimal bandwidths simultaneously at {\em all} locations. 

    \item We suggest computationally efficient versions of our method and show 
comparisons with several alternative locally adaptive methods such as 
trend filtering and wavelet thresholding. It appears that our method is competitive 
and performs significantly better than these existing methods for many types 
of signals. This leads us to believe that the proposed method is a viable and 
useful addition to the toolbox of locally adaptive nonparametric regression 
methods. We have developed an accompanying \textbf{R} package named \texttt{laser} 
which comes with a ready-to-use reference implementation of our method. 
\end{itemize}

Before concluding this section, let us point out another potential significant 
advantage of our {\em loss-function based} approach, namely that it 
naturally lends itself to other types of regression problems. Indeed, the 
squared error loss function can be replaced with more suitable loss functions 
specific to the problem at hand. For instance, we may use the $\ell_2$-loss instead of 
the squared $\ell_2$-loss as in {\em square-root lasso} \cite{belloni2011square} 
which can potentially get rid of the dependence of $\lambda$ on the noise parameter $\sigma$. We can also 
consider robustified loss functions like the quantile loss function or the 
Huber's loss function. Additionally, %
the proposed method is naturally extendable to multivariate settings. We hope to return to some of these in future works.

Overall, we believe that we provide a new take on the age old problem of optimal bandwidth selection in nonparametric regression; offering new conceptually pleasing viewpoints and insights along the way.

\subsection{Outline}
In Section~\ref{sec:results}, we formally introduce our method which we dub 
\laser{} (Locally Adaptive Smoothing Estimator for Regression) for 
convenience and discuss its properties culminating with the associated risk 
bounds in Theorem~\ref{thm:main}. In Section~\ref{sec:proof}, we prove our main result, i.e., Theorem~\ref{thm:main}. Section~\ref{sec:algsimu} is 
dedicated to computational aspects of \laser{} and simulation studies. In \S\ref{sec:pseudo}, we give a pseudo code for \laser{} as well as a computationally faster variant with comparable performance and provide a detailed analysis of their computational complexities.
In \S\ref{sec:simu}, we compare \laser{} with several popular alternative nonparametric methods via numerical experiments. We conclude with a very 
brief discussion on possible extensions of our method in different directions in Section~\ref{sec:Discussions}.

\subsection{Notation and conventions}\label{subsec:notation}
We use $[n]$ to denote the set of positive integers $\{1,2,\dots,n\}$ and $\bkt{a, b}$ to denote the (integer) interval $\{a, a + 1,\dots, b \}$ for any $a, b \in \Z$. The (real) interval $\{x \in \R: a \le x \le b\}$, where $a, b \in \R$, is denoted using the standard notation $[a, b]$. We use, in general, the bold faced $\mathbf{I}$ (with or without subscripts) to indicate a real interval, like $[0, 1]$, and $I$ to denote an integer interval, like $[n]$ or a subset thereof. In the sequel, whenever we 
speak of a {\em sub-interval of $I$} (respectively $\mathbf I$), where $I$ is an integer (respectively a real) interval, it is implicitly understood to be an integer (respectively a real) interval. For a real interval $\mathbf I$, we denote its length by $|\mathbf I|$ whereas for a {\em subset} $I$ of $[n]$, we use $|I|$ to denote its cardinality, i.e., the number of elements in $I$. Their particular usage would always be clear from the context.

For any subset $I$ of $[n]$ and $\theta = (\theta_i)_{i \in [n]}\in \R^{n}$, we let $\theta_{I} = (\theta_i)_{i \in I} \in \R^{I}$ denote its restriction to $I$. The space $\R^I$ can be canonically identified with $\R^{|I|}$ by mapping the $j$-th smallest element in $I$ to the $j$-th coordinate of vectors in $\R^{|I|}$.

In this article, we work extensively with discrete polynomial vectors. To this end, given any non-negative integer $r$ and a  subset $I$ of $[n]$, we let $S^{(r)}_I$ denote the linear subspace of discrete polynomial vectors of degree $r$ on the interval $I$, i.e.,
\begin{equation}\label{def:SIr}
    S^{(r)}_I = \Big\{\theta \in \R^{I}: \theta_i = \sum_{0 \le k \le r} a_k (\tfrac{i}{n})^k \text{ for all } i \in I  \text{ and } (a_k)_{0 \le k \le r} \in \R^{r + 1} \Big\}.
\end{equation}
We denote by ${\proj}^{(r)}_I$ the orthogonal projection onto the subspace $S^{(r)}_I$. Identifying $\R^I$ with $\R^{|I|}$ as in the previous paragraph, ${\proj}^{(r)}_I$ corresponds to a matrix of order $|I| \times |I|$. We will make this identification several times in the sequel without being explicit.

We say that a sequence of events $(E_n)_{n \ge 1}$, indexed by $n$ and  possibly depending on degree $r$ as a parameter, occurs {\em with (polynomially) high probability} (abbreviated as {\em w.h.p.}) if $\P[E_n] \ge 1 -  n^{-2}$ for all sufficiently large $n$ (depending at most on $r$). The exponent $2$ is of course 
arbitrary as we can choose {\em any} large constant by altering the values of the 
constants in our algorithm (see Theorem~\ref{thm:main} below).

Throughout the article, we use $c, C, c', C', \ldots$ to denote finite, 
positive constants that may change from one instance to the next. Numbered 
constants are defined the first time they appear and remain fixed 
thereafter. All constants are assumed to be absolute and any dependence on 
other parameters, like the degree $r$ etc. will be made explicit in 
parentheses. We prefix the subsections with \S~while referring to them.

\section{Description of \laser{} and risk bounds}\label{sec:results}
In this section, we introduce \laser{} formally, detailing the development of the estimator as a local bandwidth selector in a step-by-step manner. In \S\ref{subsec:reg_tree} we discuss the connection with Regression Trees and how it motivates \laser{}. An informal explanation for the local adaptivity of our method is given in \S\ref{subsec:intuition} aided by an illustration on a very simple yet interesting signal. Finally in  \S\ref{subsubsec:rickbnd}, we state risk bounds for \laser{} when the underlying signal is a realization of a locally H\"{o}lder regular function.

\subsection{Formal description of the method}\label{subsec:algo}
We will perform local polynomial regression of some fixed degree $r$  with a data driven bandwidth. To achieve local adaptivity w.r.t. the  regularity of the underlying signal around each point, the main issue is how to  select the bandwidth adaptively across different locations. We now  describe a way of setting the bandwidth at any given point. Let us recall from the introduction the fixed (equispaced) design model
\begin{equation}\label{eq:ythetaep}
    y_i = \theta^*_i + \epsilon_i = f(\tfrac{i}{n}) + \epsilon_i,\: 1\le i \le 
n,
\end{equation} 
where $f: [0, 1] \to \R$ is the underlying regression function and $\ep_i$'s are independent mean zero sub-Gaussian variables (see~\eqref{eq:thetaep}-\eqref{eq:psi2norm}) with sub-Gaussian norm bounded by $\sigma > 0$.

Let us recall the orthogonal projections $\Pi_I^{(r)}$ onto spaces of polynomial vectors of degree $r$ from around \eqref{def:SIr}. Now for any $I \subset [n]$ and a partition of $I$ into sets $I_1$ and $I_2 = I \setminus I_1$ and any vector $\theta \in \R^n$, let us define
			\begin{equation}\label{def:Q}
            \begin{split}
				Q^{(r)}(\theta; I_1, I) &= \|\theta_{I} - \proj^{(r)}_I \theta_{I}\|^2 -  \|\theta_{I_1} - \proj^{(r)}_{I_1} \theta_{I_1}\|^2 -  \|\theta_{I_2} - \proj^{(r)}_{I_2} \theta_{I_2}\|^2 \\&= \|\proj^{(r)}_{I_1} \theta_{I_1}\|^2 + \|\proj^{(r)}_{I_2} \theta_{I_2}\|^2 - \|\proj^{(r)}_{I} \theta_{I}\|^2 = \|\Pi_{I_1, I_2}^{(r)}\theta_I\|^2
                \end{split}
			\end{equation}            
where $\|\eta\| \stackrel{{\rm def.}}{=} (\sum_{j \in J} \eta_j^2)^{\frac12}$ denotes the usual $\ell_2$-norm of any $\eta \in \R^J$ $(J \subset [n])$ and $\Pi^{(r)}_{I_1, I_2}$ is the orthogonal projection onto the subspace ${S_{I}^{(r)}}^{\perp} \cap \big({S_{I_1}^{(r)} \oplus S_{I_2}^{(r)}}\big)$ (note in view of \eqref{def:SIr} that $S_{I}^{(r)}$ is a subspace of $S_{I_1}^{(r)} \oplus S_{I_2}^{(r)}$). Consequently, $Q^{(r)}$ (for every fixed $I$ and $I_1$) is a positive semi-definite quadratic form on  $\R^n$. We will be interested in the case where $I$ and $I_1$ are sub-intervals of $[n]$.

Next we introduce what we call a {\em local discrepancy measure}. For any $\theta \in \R^n$, let us introduce the associated {\em ($r$-th order) local discrepancy measure} on sub-intervals of $[n]$ as follows.
\begin{equation}\label{def:discrep}
				T^{(r)}_{\theta}(I) \stackrel{{\rm def.}}{=} 
                \max_{I_1,I_2}
                \sqrt{Q^{(r)}(\theta; I_1,I)}
			\end{equation}
where $\{I_1, I_2\}$ range over all partitions of $I$ into an interval $I_1$ and its complement. This definition is legitimate as $Q^{(r)}$ is positive semi-definite. Intuitively, one can think of $T^{(r)}_{\theta}(I)$ as a measure of deviation of $\theta$ from the subspace of degree $r$ polynomial vectors on the interval $I$. If $\theta$ is exactly a polynomial of degree $r$ on $I$, then $T^{(r)}_{\theta}(I) = 0$.

We now come to the precise description of our estimator. Given any location $i_0 \in [n]$ and a {\em bandwidth} $h \in \N$, let us consider the \emph{truncated} symmetric interval 
\begin{equation}\label{def:bkt}
    \bkt{i_0 \pm h} = \bkt{i_0 \pm h}_n \stackrel{{\rm def.}}{=} \bkt{(i_0 - h)\vee 1, (i_0 + h)\wedge n} \, (\subset [n]).
\end{equation} 

Our idea is to choose $I = \bkt{i_0 \pm h}$ as a potential 
interval for estimating $\theta_{i_0}^\ast$ if the local discrepancy measure $T^{(r)}_y(I)$ is small. Just checking that $T^{(r)}_y(I)$ is small is of course not enough; for instance, the singleton interval $\{i_0\}$ will have $T^{(r)}_y(\{i_0\}) = 0$. Naturally, {we are led to choosing the largest symmetric interval $I$ around $i_0$ for which $T^{(r)}_y(I)$ is still small}. To this end, let us define a threshold $\lambda \in (0, \infty)$ which would be the {\em tuning parameter} in our method. For any such threshold $\lambda$, we define the set of ``good'' bandwidths as 
\begin{equation}\label{def:good_bandwidth}
    \mathcal{G}^{(r)}(\lambda, y) = \{h \in \bkt{0, n-1} : T^{(r)}_{y}(\bkt{i_0 \pm h}) \leq \lambda\}.
\end{equation}
We now propose our optimal local bandwidth as follows.
\begin{equation}\label{def:bandwidth}
    \hat{h}_{i_0} = \hat{h}_{i_0}^{(r)}(\lambda, y) \stackrel{{\rm def.}}{=} \max \mathcal{G}^{(r)}(\lambda, y).
\end{equation}
With this choice of optimal bandwidth, our proposed estimator for $f(\tfrac{i_0}{n}) = \theta^\ast_{i_0}$ (of degree $r$) takes the following form:
\begin{equation}\label{eq:defn}
    \hat f(\tfrac{i_0}{n}) = {\hat f}_{\laser(r, \lambda)}(\tfrac{i_0}{n}) = \big(\proj^{(r)}_{\bkt{i_0 \pm \hat h_{i_0}}} y_{\bkt{i_0 \pm \hat{h}_{i_0}}}\big)_{i_0}.
\end{equation}
    
\subsection{Connection to Regression Trees}\label{subsec:reg_tree}
Our estimator is naturally motivated by the splitting criterion used in Regression Trees. In this section, we explain this connection. Note that trees are in one to one correspondence with partitions of $[n]$ in the univariate setting. One can define the \textit{Optimal Regression Tree }(ORT) estimator, described in~\cite{chatterjee2021adaptive} as a solution to the following penalized least squares problem:
\[
    \hat{\theta}_{{\rm ORT},\,\lambda}^{(r)} = \argmin_{\theta \in \R^n} \big(\|y - \theta\|^2 + \lambda k^{(r)}(\theta)\big),
\]
where $k^{(r)}(\theta)$ denotes the smallest positive integer $k$ such that if we take a partition of $[n]$ into $k$ intervals $I_1,\dots,I_k$ then the restricted vector $\theta_{I_j}$ is a degree $r$ (discrete) polynomial vector on $I_j$ for all $1 \leq j \leq k$. The version with $r = 0$ is also called jump/$\ell_0$-penalized least squares or the Potts functional minimizer; see~\cite{boysen2009consistencies}. The final tree produced is a random partition $\mathcal{P}^{(r)}$ and the final fit is obtained by performing least squares degree $r$ polynomial regression on each interval of the partition $\mathcal{P}^{(r)}$. 

We now make a key observation. If we split a resulting interval $I$ of the final ``tree'' $\mathcal{P}^{(r)}$ further into \textit{any} two intervals; one does not decrease the objective function. This turns out to be equivalent to saying that the decrease in residual sum of squares is less than a threshold (the tuning parameter) $\lambda.$ Let us call this property $({\rm P}^*)$ which the interval $I$ satisfies. 

The gain in residual sum of squares %
when splitting $I$ into two contiguous intervals $I_1,I_2$ is precisely $Q^{(r)}(y;I_1,I)$ defined 
in~\eqref{def:Q}. In view of property $({\rm P}^*)$, we see that $I$ satisfies
\[
    \max_{I_1,I_2} Q^{(r)}(y;I_1,I) \leq \lambda,
\]
where $I_1,I_2$ ranges over all splits of $I$ into two contiguous intervals. The above display naturally leads us to define the local discrepancy measure $T_{y}^{(r)}(I)$ as in~\eqref{def:discrep}. The only difference is instead of maximizing over all splits, we insist on $I_1$ being any subinterval of $I$ and $I_2 = I \cap I_1^c$ not necessarily an interval. We found that this modification simplifies our proof substantially. 

Our idea to produce locally adaptive fits is to now execute the following principle. 
For any given point, choose the largest interval containing this point which satisfies property $({\rm P}^*)$ and estimate by the mean (or higher order regression) within this interval. In effect, motivated by the splitting criterion for Regression Trees, we are proposing a principled way to perform adaptive bandwidth selection in local polynomial regression.

\subsection{Local adaptivity of \laser{} in a toy example}\label{subsec:intuition}
Intuitively, it is perhaps clear that our bandwidth~\eqref{def:bandwidth} at a given location $i_0$ is larger when $\theta^*$ is ``smoother'' around $i_0$. This smoothness is measured by the local discrepancy statistic $T^{(r)}_{y}(I)$ for various intervals $I$ centered at $i_0$. The following quantity turns out to play the role of an ``effective noise'' in our problem.
\begin{equation*}
    {\rm NOISE} \stackrel{{\rm def.}}{=} \max_{I}
    T^{(r)}_{\ep}(I).
\end{equation*}
Using concentration inequalities involving sub-Gaussian variables it can be shown 
that, this effective noise does not exceed $C \sigma \sqrt{\log n}$ w.h.p. (see 
Lemma~\ref{lem:noise_bnd} in Section~\ref{sec:proof}). In particular, the effective noise acts as an upper bound on the contribution from the 
noise that is {\em uniform} across all intervals $I \subset [n]$. Combining this with the seminorm property of $\sqrt{Q^{(r)}}$, one gets a 
valuable information on the selected bandwidth $\hat h_{i_0}$ in view of our selection rule in \eqref{def:good_bandwidth}--
\eqref{def:bandwidth}, namely
\begin{equation}\label{eq:noise_removal}
T^{(r)}_{\theta^*}(\bkt{i_0 \pm \hat{h}_{i_0}}) \asymp \sigma \sqrt{\log n}
\end{equation}
w.h.p. simultaneously for all $i_0 \in [n]$ as long as the tuning parameter $\lambda$ \textit{kills} the effective noise, as in, e.g.,
\[
    \lambda = 2\: {\rm NOISE} = C \sigma \sqrt{\log n}.
\]
Here ``$\asymp$'' in \eqref{eq:noise_removal} means that the ratio of both sides 
stays bounded away from $0$ and $\infty$. We subsequently refer to 
\eqref{eq:noise_removal} as the \textit{bandwidth selection equation}. %
See Proposition~\ref{prop:noise_threshold_decomp} in Section~\ref{sec:proof} for a 
precise formulation. 

In effect, \eqref{eq:noise_removal} says that if $\lambda = C \sigma \sqrt{\log n}$ is chosen so that it exceeds the effective noise level, then LASER selects the 
bandwidths resembling the following oracle. The oracle can see the signal 
$\theta^\ast$ itself. For every location $i_0$, the oracle starts with the smallest 
bandwidth $h = 0$ and continues to increase $h$. At each step, the oracle calculates 
the local discrepancy measure $T^{(r)}_{\theta^*}(\bkt{i_0 \pm h})$ and stops the 
first time it goes above $C \sigma \sqrt{\log n}$ to output the selected bandwidth 
at $i_0$. What is very crucial is that the stopping threshold is universal in the
sense that it does not depend on the location $i_0$ nor the underlying signal
$\theta^*$. 

We illustrate the importance of this observation with a simple yet illuminative example. To this end let us consider the function $f_{\texttt{Check}}(x) = (x - \frac12)1\{x \geq \frac12\}$. The signal version, i.e., the corresponding $\theta$ is $\theta^*_i = \frac{i - 1}{2}1\{i \ge n/2\}$.

Let us now examine local averaging which is local polynomial regression of degree $0$, i.e.,
\[
    \hat{\theta}_{i_0}(h) = \overline{y}_{\bkt{i_0 \pm h}}
\]
where $h > 0$ is some bandwidth. For any $h > 0$, one can compute explicitly the bias and 
variance of $\hat{\theta}_{i_0}(h)$ as a function of $h$. One can then check that the ideal bandwidths for any point in $[0,0.5)$ and any point in $[0.5,1]$ are $c n$ and $c n^{2/3}$ respectively. The ideal squared error rates turn out to be at most $C n^{-1}$ and $C n^{-2/3}$ respectively. So, even in this simple example, one needs to set different bandwidths in different locations to get the best possible rates of convergence.

Let us now check if the oracle selects the right bandwidths in this example. Take a point in $[0.5,1]$ such as $x = \frac{3}{4}$, i.e., $i_0 = \frac{3n}{4}$. Consider intervals of the form $I(h) = \bkt{i_0 \pm h}$. We need to find $h$ which solves the Bandwidth Selection Equation \eqref{eq:noise_removal} (with $h$ in place of $\hat h_{i_0}$). In the case of local averaging with $r = 0$, %
$Q$ admits of a simplified expression as follows:
\[
    Q^{(0)}(\theta^\ast; I_1,I) =  \frac{|I_1| |I_2|}{|I|}  \big(\overline{\theta}^{\ast}_{I_1} -  \overline{\theta}^{\ast}_{I_2}\big)^2,
\]
where $I_2 = I \setminus I_1$.

It turns out in this case, that the local discrepancy $T^{(0)}_{\theta^\ast}(I(h))$ is maximized when $I_1$ and $I_2$ are roughly of size $h/2$. Clearly, for such a pair $(I_1, I_2)$, one has
\[
    \frac{|I_1| |I_2|}{|I|} \asymp c h \text{ and } \big(\overline{\theta}_{I_1}^\ast -  \overline{\theta}_{I_2}^\ast\big) \asymp c\frac{h}{n}.
\]
This implies that
\[
    T^{(r)}_{\theta^\ast}(\bkt{i_0 \pm h}) \asymp c\frac{h^3}{n^2}.
\]
Thus, for solving~\eqref{eq:noise_removal} we need $h \asymp n^{2/3}$ which is exactly 
the right order of the bandwidth for this $i_0$.

Now, let us consider a point in $(0,0.5)$ such as $i_0 = \frac{3n}{8}.$ It is clear that if $h \leq \frac{n}{8}$ then $T^{(r)}_{\theta^*}([i_0 \pm h]) = 0$, hence the selected bandwidth is not less than $\frac{n}{8}$. This means that $h \asymp n$ which is exactly the right bandwidth size (in order) for this $i_0.$ 

To summarize, we find that solving the \emph{same} bandwidth selection 
equation~\eqref{eq:noise_removal} gives the correct bandwidth for locations 
\textit{both} in the left and right half of the domain. This illustration suggests 
that an $h$ satisfying \eqref{eq:noise_removal} is potentially the right bandwidth 
to select even for general degrees $r \geq 0$. It turns out that this intuition is 
correct and LASER precisely implements the above bandwidth selection rule. The 
underlying reason why this bandwidth selection rule works is due to the \textit{self-adaptive} growth rate of the local discrepancy measure $T^{(r)}_{\cdot}(\cdot)$ 
(see, e.g., Lemma~\ref{lem:Tr_bnd_Holder}) of which we have already seen some indication in 
the case of $r = 0$.

Our proof in Section~\ref{sec:proof} gives a unified analysis for all degrees $r \geq 0$. Since there does not seem to be a ``simple'' expression for the term $Q^{(r)}$ for higher degrees $r > 0$, the general case turns out to be more subtle. Our proof in Section~\ref{sec:proof} reveals that  the calibrated bandwidth obtained by \laser{}, as a solution to the bandwidth selection equation \eqref{eq:noise_removal}, leads to an automatic and correct balancing of the local bias and variance terms yielding the desired property of local adaptivity.

\subsection{Pointwise risk bounds for {\laser}}\label{subsubsec:rickbnd}
For theoretical risk bounds and the accompanying in-depth mathematical analysis of {\laser}, we choose to work with (locally) H\"{o}lder regular functions which have a long history in nonparametric regression, see, e.g., \cite{lepskii1991problem}, \cite{DonohoJohnKeky95} and \cite{LMS97}. Let us formally introduce this class of functions with a slightly non-standard notation for our convenience.

\begin{definition}[H\"{o}lder space]\label{def:holder}
Given any (open) sub-interval $\mathbf I$ of $[0, 1]$, $\alpha \in [0, 1]$ and $r \ge 0$ an integer, we define the H\"{o}lder space $C^{r, \alpha}(\mathbf I)$ as the class of functions $f : [0, 1] \to \R$ which are $r$-times continuously differentiable on $\mathbf I$ and furthermore the $r$-th order derivative $f^{(r)}$ is H\"{o}lder continuous with exponent $\alpha$, i.e.,
\begin{equation}\label{eq:Holder}
   |f|_{\mathbf I; r, \alpha} \stackrel{{\rm def.}}{=} \sup_{x, y \in \mathbf I, x \ne y} \frac{|f^{(r)}(x) - f^{(r)}(y)|}{|x - y|^{\alpha}} < \infty. 
\end{equation}
We call $|f|_{\mathbf I; r, \alpha}$ the $(r, \alpha)$-{\em H\"{o}lder coefficient} (or {\em norm}) of $f$ on $\mathbf I$. Notice that if \eqref{eq:Holder} holds for some $\alpha > 1$, then $|f|_{\mathbf I; r, \alpha}$ is necessarily $0$, i.e., $f^{(r)}(\cdot)$ is constant and consequently $f$ is a polynomial of degree $r$ on $\mathbf I$. For the sake of continuity, we denote the space of such functions by $C^{r, \infty}(f)$ and set $|f|_{\mathbf I; r, \infty} = 0$.
\end{definition}

Our main result (see Theorem~\ref{thm:main} below) in this paper shows that \laser{} 
adapts near-optimally to the local H\"{o}lder coefficient as well as norm of the 
underlying true signal $f$. In the sequel, for any $x \in [0,1]$ and $s > 0$, we let 
\begin{equation}\label{def:bkt2}
    [x \pm s] \stackrel{{\rm def.}}{=}[(x-s)\vee 0, (x + s) \wedge 1] \, (\subset [0, 1])
\end{equation}
(cf.~\eqref{def:bkt}).
\begin{theorem}[Local Adaptivity Result]
\label{thm:main}
Fix a degree $r \in \N$ and let $f : [0, 1] \to \R$. There exist constants $\Cl{C:lambda}$ and $\Cl{C:bnd} = \Cr{C:bnd}(r)$ such that the following holds with high probability for $\lambda = C_1 \sigma \sqrt{\log n}$. Simultaneously for all quadruplets $(i_0, s_0, r_0, \alpha_0)$ where $i_0 \in [n]$, $s_{0} \in (0, 1)$, $r_0 \in [0, r]$ an integer and $\alpha_0 \in [0, 1] \cup \{\infty\}$ such that $f \in C^{r_0, \alpha_0}([\frac{i_0}n \pm s_0])$, one has, with $\alpha = \alpha_0 + r_0$, 
\begin{equation}\label{eq:main}
	|\hat{f}(\tfrac{i_0}n) - f(\tfrac{i_0}{n})| \leq \Cr{C:bnd} \, \, \big(\sigma^{\frac{2\alpha}{2\alpha + 1}}\,|f|_{[\frac{i_0}n \pm s_0]; r_0, \alpha_0}^{\frac1{2\alpha + 1}}\, (\tfrac{\log n}{n})^{\frac{\alpha}{2\alpha + 1}} + \sigma \, (\tfrac{\log n}{ns_0})^{\frac12}\big),
\end{equation}
where $\hat{f}(\frac{i_0}n) = \hat{f}_{\laser{}(r, \lambda)}
(\tfrac{i_0}n)$ is from \eqref{eq:defn} and we interpret $0^0 = 0$.
\end{theorem}

Let us now discuss some aspects of the above theorem.
\begin{itemize}
    \item Our bound achieves near optimal sample complexity. For instance, if $f \in C^{r_0, \alpha_0}([0, 1])$ is globally H\"{o}lder continuous and we ignore the dependence on the H\"{o}lder coefficient of $f$ or the noise strength $\sigma$, then the risk bound in \eqref{eq:main} reads as $C(r)(\frac{\log n}{n})^{\frac{\alpha}{2\alpha + 1}}$ which is known to be the minimax optimal rate up to logarithmic factors (see, e.g.,~\cite{DonohoJohnKeky95}).

    \item We are mainly interested in the cases where the H\"older exponents are different at different points $i_0$ in the domain, i.e., $r_0, \alpha_0$ can depend on $i_0$. One can think of $s_0$ as typically $O(1)$ in any reasonable example. The main point we emphasize here is that our bound at different points $i_0$ adapts optimally to $r_0, \alpha_0$ simultaneously under the \textit{same} choice of $\lambda.$ The first term gives the optimal rate up to logarithmic factors and the second term gives a parametric rate and hence is a lower order term. 

    \item The degree $r$ of the estimator is chosen by the user. Once chosen, \laser{} adapts to any local H\"older degree $r_0 \leq r$ and any Holder exponent $\alpha_0 \in [0,1] \cup \{\infty\}.$ 
    
    \item The logarithmic factor is known to be necessary if one wants to adapt to all levels of H\"older exponent $\alpha_0 \in [0, 1]$. It appears that the logarithmic factor $(\log n)^{\frac{\alpha_0 + r_0}{2(\alpha_0 + r_0) + 1}}$ that we incur in \eqref{eq:main} is the best possible (see~\citep{lepskii1991problem}).

    \item The case $\alpha_0 = \infty$ is particularly interesting. Let us recall from Definition~\ref{def:holder} that $f$ is locally a polynomial of degree (at most) $r$ in this case so that $|f|_{[\frac{i_0}n \pm s_0; r, \infty]} = 0$. Consequently, we recover the parametric rate from \eqref{eq:main}.

    \item Although we are unaware of any result on the optimal dependence of the risk in terms of the H\"older coefficient $|f|_{[\frac{i_0}n \pm s_0; r_0, \alpha]}$, it is clear that our bound gets better for smoother functions with smaller H\"{o}lder coefficients. 
\end{itemize}

\section{Proof of the main result}\label{sec:proof}
Our proof of Theorem~\ref{thm:main} proceeds through four distinct stages. Revisit \S\ref{subsec:algo} to recall the relevant definitions.

\smallskip

\noindent {\bf Stage $1$: A bias-variance type decomposition.} We can write the 
estimation error as
\begin{equation}\label{eq:bias_noise0}
\begin{split}
\big|\hat f(\tfrac{i_0}n) - \theta_{i_0}^\ast\big| &\stackrel{\eqref{eq:defn}}{=} 
\big|\big(\proj^{(r)}_{\bkt{i_0 \pm \hat h_{i_0}}} y_{\bkt{i_0 \pm 
\hat{h}_{i_0}}}\big)_{i_0} - \theta_{i_0}^\ast\big| \\
&\stackrel{\eqref{eq:ythetaep}}{=} \big| (\big(\proj^{(r)}_{\bkt{i_0 \pm \hat 
h_{i_0}}} \theta^\ast_{\bkt{i_0 \pm \hat{h}_{i_0}}}\big)_{i_0} - 
\theta_{i_0}^\ast) + \big(\proj^{(r)}_{\bkt{i_0 \pm \hat h_{i_0}}} 
\epsilon_{\bkt{i_0 \pm \hat{h}_{i_0}}}\big)_{i_0} \big| \\
& \,\, \leq \big| (\big(\proj^{(r)}_{\bkt{i_0 \pm \hat 
h_{i_0}}} \theta^\ast_{\bkt{i_0 \pm \hat{h}_{i_0}}}\big)_{i_0} - 
\theta_{i_0}^\ast) \big| + \big| \big(\proj^{(r)}_{\bkt{i_0 \pm \hat h_{i_0}}} 
\epsilon_{\bkt{i_0 \pm \hat{h}_{i_0}}}\big)_{i_0} \big|,
\end{split}
\end{equation}
As we explain below, we have
\begin{equation}\label{Piep_bnd}
\big| \big(\proj^{(r)}_{\bkt{i_0 \pm \hat h_{i_0}}} 
\epsilon_{\bkt{i_0 \pm \hat{h}_{i_0}}}\big)_{i_0} \big| \le C(r) \sigma \sqrt{\frac{\log n}{\hat{l}_{i_0}}} \quad \mbox{w.h.p.}
\end{equation}
where $\hat{l}_{i_0}$ denotes the length of the (random) interval $\bkt{i_0 \pm 
\hat{h}_{i_0}}$ (which may be different from $2\hat{h}_{i_0}$ in view of 
Definition~\ref{def:bkt}). Combined with \eqref{eq:bias_noise0}, this implies
\begin{equation}\label{eq:bias_noise}
\begin{split}
\big|\hat f(\tfrac{i_0}n) - \theta_{i_0}^\ast\big| \le \big| (\big(\proj^{(r)}_{\bkt{i_0 \pm \hat 
h_{i_0}}} \theta^\ast_{\bkt{i_0 \pm \hat{h}_{i_0}}}\big)_{i_0} - 
\theta_{i_0}^\ast) \big| + C(r) \sigma \sqrt{\frac{\log n}{\hat{l}_{i_0}}} =: {\rm 
B}_{i_0} + {\rm N}_{i_0} \,\,\,\, \mbox{w.h.p.}
\end{split}
\end{equation}

Let us now verify \eqref{Piep_bnd}. Fix any interval $I \subset [n]$ containing $i_0$. We can write $\big(\proj^{(r)}_{I} \epsilon_{I}\big)_{i_0} = \sum_{j \in I}
(\proj^{(r)}_{I})_{i_0, j} \epsilon_j$ as a linear combination of $\{\epsilon_j: j \in 
I\}$ where $(\proj^{(r)}_{I})_{i, j}$ denotes the $(i, j)$-th element of the matrix 
corresponding to $\proj^{(r)}_{I}$ (see below \eqref{def:SIr}). Since $\ep_j$'s are 
independent sub-Gaussian variables with sub-Gaussian norm bounded by $\sigma$ (recall \eqref{eq:psi2norm}), it 
follows from a standard application of the Cauchy-Schwarz inequality that the sub-
Gaussian norm of $\big(\proj^{(r)}_{I} \epsilon_{I}\big)_{i_0}$ is bounded by $\sigma 
\sqrt{\sum_{j \in I} (\proj^{(r)}_{I})_{i_0, j}^2}$. Now note that 
\begin{equation*}
\sum_{j \in I} (\proj^{(r)}_{I})_{i_0, j}^2 = \sum_{j \in I} (\proj^{(r)}_{I})_{i_0, j} (\proj^{(r)}_{I})_{j, i_0} = (\proj^{(r)}_{I})_{i_0, i_0}^2 = (\proj^{(r)}_{I})_{i_0, i_0}.
\end{equation*}
In the first equality we used the fact that the orthogonal projection matrix 
$\proj^{(r)}_{I}$ is symmetric whereas in the last equality we used that it is 
idempotent. Next, using a property about the subspace of discrete polynomials stated 
in Lemma~\ref{lem:poly1} below, we obtain
$$(\proj^{(r)}_{I})_{i_0, i_0} \leq \frac{C(r)}{|I|}.$$
Therefore, $\sqrt{|I|}\big(\proj^{(r)}_{I} \epsilon_{I}\big)_{i_0}$ is a sub-Gaussian variable with sub-Gaussian norm bounded by $C(r) \sigma$ for {\em any} 
interval $I$ containing $i_0$. Since the number of such intervals is at most $n^2$, 
it follows from standard results on the extrema of sub-Gaussian random variables 
(see, e.g., \cite[Exercise~2.12]{wainwright2019high}) that 
\begin{equation*}
\sup_{I}\sqrt{|I|}\big(\proj^{(r)}_{I} \epsilon_{I}\big)_{i_0} \le C(r)\sigma \sqrt{\log n} \quad \mbox{w.h.p.}
\end{equation*}
whence \eqref{Piep_bnd} follows.

Now going back to the bound \eqref{eq:bias_noise}, one can think of ${\rm B}_{i_0}$ and ${\rm N}_{i_0}$ as the bias and variance (standard deviation) components of the 
estimation error respectively {\em if} we disregard the randomness of $\hat{h}_{i_0}$. The bias term ${\rm B}_{i_0}$ would generally become 
larger as $\hat{l}_{i_0}$ increases, whereas the variance term would decrease. We shall 
explicitly bound the bias and variance terms separately at later stages. For this, 
we first need good (deterministic) upper and lower bounds on the bandwidth 
$\hat{h}_{i_0}$ or equivalently the length $\hat{l}_{i_0}$. The first step towards 
this is the bandwidth selection equation which we informally introduced in 
\eqref{eq:noise_removal}.

\smallskip

\noindent {\bf Stage $2$: Bandwidth selection ``equation''.} The following 
proposition governs our selected bandwidths. 
\begin{proposition}\label{prop:noise_threshold_decomp}
There exists an absolute constant %
$\Cl{C:noise} \in (0, \infty)$ such that for any %
$\lambda \in (0, \infty)$, we have
\begin{equation}\label{eq:noise_threshold_decomp_ub}
\begin{split}
&T^{(r)}_{\theta^\ast}(\bkt{i_0 \pm \hat h_{i_0}}) \le \lambda + 
\Cr{C:noise}\sigma\sqrt{\log n}
\end{split}
\end{equation}
w.h.p. simultaneously for all $i_0 \in [n]$ where $\hat h_{i_0} = 
\hat{h}_{i_0}^{(r)}(\lambda, y)$ is %
from \eqref{def:bandwidth}. %
Furthermore, unless $\bkt{i_0 \pm \hat h_{i_0}} = [n]$, we also have
\begin{equation}\label{eq:noise_threshold_decomp_lb}
T^{(r)}_{\theta^\ast}(\bkt{i_0 \pm (\hat h_{i_0} + 1)}) \ge \lambda - 
\Cr{C:noise}\sigma\sqrt{\log n}
\end{equation}
w.h.p. simultaneously for all $i_0 \in [n]$.
\end{proposition}
The inequalities \eqref{eq:noise_threshold_decomp_ub}--\eqref{eq:noise_threshold_decomp_lb} tell us that while the bandwidth is selected as per \eqref{def:bandwidth} by choosing the largest interval $I$ whose local discrepancy $T^{(r)}_{y}(I)$ w.r.t. {\em the observation vector $y$} is at most $\lambda$, its local discrepancy $T^{(r)}_{\theta^\ast}(I)$ w.r.t. {\em the underlying signal} $\theta^*$ is ``almost'' equal to $\lambda$ (in order) provided the latter 
exceeds the effective noise level $C \sigma \sqrt{\log n}$ (see 
Lemma~\ref{lem:noise_bnd} and Remark~\ref{rmk:noise_bnd} below). %
This reveals a key \emph{self-normalization} property of the quantity 
$T^{(r)}_{\theta^\ast}(\bkt{i_0 \pm \hat h_{i_0}})$ in the sense that it does not 
depend on the signal $\theta^*$, the location $i_0$ or the width $\hat h_{i_0}$. 
This property is thus extremely useful for obtaining suitable bounds on $\hat 
h_{i_0}$ (or $\hat l_{i_0}$) and is the main driver of the local adaptivity of our 
method as we will see in the upcoming stages.

Let us now give the proof of Proposition~\ref{prop:noise_threshold_decomp} which 
requires some preparation. Let us recall the quadratic form $Q^{(r)}(\theta; I_1, I)$
from \eqref{def:Q} where $I_1 \subset I \subset [n]$ (not necessarily 
intervals). Also recall the definition of the local discrepancy measure 
$T^{(r)}_{\theta}(I)$ (of degree $r$) from \eqref{def:discrep}. The following lemma shows that $T^{(r)}_{\epsilon}(I)$ is small uniformly over 
$I$ w.h.p. where $\epsilon = (\epsilon_i)_{i \in [n]} \in \R^n$ is the vector of 
noise from \eqref{eq:ythetaep}.

\begin{lemma}\label{lem:noise_bnd}
We have,
\begin{equation*}%
\max_{I} T_{\epsilon}^{(r)}(I) \le C %
\sigma \sqrt{\log n} \:\: \text{w.h.p.}
\end{equation*}
where $I$ ranges over all sub-intervals of $[n]$ and $\ep$ is as in 
\eqref{eq:ythetaep}.
\end{lemma}
\begin{remark}\label{rmk:noise_bnd}
The quantity $\max_{I} T^{(r)}_{\epsilon}(I)$ plays the role of \emph{effective noise} in our analysis and the bound $C \sigma \sqrt{\log n}$ thus is the effective noise level in our problem.    
\end{remark}

\begin{proof}
Consider a sub-interval $I$ of $[n]$ and a partition of $I$ into $I_1$ and $I_2$ where $I_1$ is an interval. From definition \eqref{def:Q}, we have
\begin{equation}\label{eq:Qformmat}
    Q^{(r)}(\theta; I_1, I) = \|\Pi_{I_1, I_2}^{(r)}\theta_I\|^2,
\end{equation}
where $\Pi_{I_1, I_2}^{(r)}$ is the orthogonal projection onto the subspace ${S_{I}^{(r)}}^{\perp}\cap \big({S_{I_1}^{(r)} \oplus S_{I_2}^{(r)}}\big)$. Assuming that $\epsilon$ is a vector of independent centered sub-Gaussian random variables with sub-Gaussian norm bounded by $\sigma$ (see \eqref{eq:psi2norm}), we obtain as a consequence of the Hanson-Wright inequality (cf. Theorem 2.1 in \cite{rudelson2013hanson}) that 
\begin{equation}\label{eq:sqrtQrttail}
    \sqrt{Q^{(r)}(\epsilon; I_1, I)} - \|\Pi_{I_1, I_2}^{(r)}\|_{{\rm Fr}} = \|\Pi_{I_1, I_2}^{(r)}\epsilon\| - \sigma\|\Pi_{I_1, I_2}^{(r)}\|_{{\rm Fr}} \text{ is sub-Gaussian}
\end{equation}
with sub-Gaussian norm bounded by $C\sigma^2 \|\Pi_{I_1, I_2}^{(r)}\|$, where, for any operator %
on $\R^I$ %
or equivalently a $|I| \times |I|$ matrix $X$ (see below \eqref{def:SIr} to recall our convention),  $\|X\|$ denotes %
the ($\ell_2$-)operator norm whereas $$\|X\|_{{\rm Fr}} \stackrel{{\rm def.}}{=}\sqrt{{\rm Tr}(X^\top X)}$$ 
is the Frobenius (or the Hilbert-Schmidt) norm.

Since $\Pi_{I_1, I_2}^{(r)}$ is an orthogonal projection, we have 
\begin{equation}\label{eq:Bop}
\|\Pi_{I_1, I_2}^{(r)}\| \le %
1.
\end{equation}
Also, 
\begin{equation}\label{eq:BFr}
\|\Pi_{I_1, I_2}^{(r)}\|_{{\rm Fr}}^2 %
= \mathrm{Tr}(%
\Pi_{I_1, I_2}^{(r)}) %
= r+1
\end{equation}
where in the final step we used the property that the trace of a 
projection (idempotent) matrix is equal to its rank.

Now using standard facts about the maxima of sub-Gaussian random variables (see the 
proof of \eqref{Piep_bnd} above) %
and %
observing that there are at most $n^4$ many %
pairs of intervals %
$(I, I_1)$ under consideration below, we obtain from the preceding 
discussions that
\begin{align*}
    \max_I T^{(r)}_{\epsilon}(I) &\stackrel{\eqref{def:discrep}}{=} \max_I \max_{\substack{I_1, I_2 \\ I_1 \cup I_2 = I}} \sqrt{Q^{(r)}(\epsilon; I_1, I)} %
    \stackrel{\eqref{eq:Qformmat}}{=} \max_I \max_{\substack{I_1, I_2 %
    }} \|%
    \Pi_{I_1, I_2}^{(r)}\epsilon\| \\
    &\le \max_I \max_{\substack{I_1, I_2 %
    }} \|%
    \Pi_{I_1, I_2}^{(r)}\|_{{\rm Fr}} + \max_I \max_{\substack{I_1, I_2 %
    }} \bigg(\|%
    \Pi_{I_1, I_2}^{(r)} \epsilon\|_{{\rm Fr}} - \|%
    \Pi_{I_1, I_2}^{(r)}\|_{{\rm Fr}}\bigg) \\
    &\stackrel{\eqref{eq:sqrtQrttail}, \eqref{eq:Bop}}{\le} C \sigma (\sqrt{r} + 1)  + C \sigma \sqrt{\log n}. %
    \qedhere
\end{align*}
\end{proof}

Now we are in a position to give the proof of Proposition~\ref{prop:noise_threshold_decomp}.

\begin{proof}[Proof of Proposition~\ref{prop:noise_threshold_decomp}]
$\sqrt{Q^{(r)}(\theta; I_1, I)}$ is a seminorm on $\R^n$ due to 
\eqref{eq:Qformmat}. %
Since $y = \theta^\ast + \epsilon$, 
we then obtain from the triangle inequality,
\begin{equation*}%
\sqrt{Q^{(r)}(y; I_1, I)} - \sqrt{Q^{(r)}(\epsilon; I_1, I)} 
\le \sqrt{Q^{(r)}(\theta^\ast; I_1, I)} \le \sqrt{Q^{(r)}(y; I_1, I)} + \sqrt{Q^{(r)}(\epsilon; I_1, I)}.
\end{equation*}
Since $\sqrt{Q^{(r)}(\epsilon; I_1, I)} \le T^{(r)}_{\ep}(I)$ by 
definition, %
we obtain
\begin{equation*}
\sqrt{Q^{(r)}(y; I_1, I)} - T^{(r)}_{\ep}(I) \le \sqrt{Q^{(r)}(\theta^\ast; I_1, I)} \le \sqrt{Q^{(r)}(y; I_1, I)}  + T^{(r)}_{\ep}(I).
\end{equation*}
Now taking maximum over all pairs $(I_1, I_2)$ that form a partition of a 
sub-interval $I$ of $[n]$ %
with $I_1$ an interval, we get
\begin{equation*}
T^{(r)}_{y}(I) - T^{(r)}_{\epsilon}(I)  \le T^{(r)}_{\theta^\ast}(I) 
\le T^{(r)}_{y}(I) + T^{(r)}_{\epsilon}(I).
\end{equation*}
Plugging the bound on the maximum of $T^{(r)}(\epsilon, I)$ from 
Lemma~\ref{lem:noise_bnd} into this display, we obtain
\begin{equation*}
T^{(r)}_{y}(I) - C\sigma\sqrt{\log n}  \le T^{(r)}_{\theta^\ast}(I) 
\le T^{(r)}_{y}(I) + C\sigma\sqrt{\log n}
\end{equation*}
w.h.p. {\em simultaneously} for all sub-intervals $I$ of $[n]$. We can now 
conclude \eqref{eq:noise_threshold_decomp_ub} and 
\eqref{eq:noise_threshold_decomp_lb} from this in view of the definition 
of $\hat h_{i_0}$ in \eqref{def:bandwidth} (and 
\eqref{def:good_bandwidth}).
\end{proof}

\noindent {\bf Stage $3$: Bounding the variance term.} We now show how one half of the bandwidth selection equation, namely 
\eqref{eq:noise_threshold_decomp_lb}, leads to a lower bound on $\hat l_{i_0}$ and 
consequently an upper bound on the variance term ${\rm N}_{i_0}= C(r) \sigma \sqrt{\frac{\log n}{\hat{l}_{i_0}}}$ in \eqref{eq:bias_noise}. For this we first 
need an upper bound on the local discrepancy $T^{(r)}_{\theta^\ast}(\bkt{i_0 \pm 
\hat h_{i_0}})$ in terms of the length $\hat{l}_{i_0}$ via the following lemma.  %

\begin{lemma}\label{lem:Tr_bnd_Holder}
Let %
$\mathbf I$ be a sub-interval of $[0, 1]$ such that $f \in C^{r_0, 
\alpha}(\mathbf I)$ %
for some $r_0 \in [0, r]$ an integer and $\alpha \in [0, 1] \cup \{\infty\}$. Then we have,
\begin{equation}\label{eq:Tr_bnd_Holder}
T^{(r)}_{\theta^\ast}\big(%
\bkt{n\mathbf I} \big) \le \frac{|f|_{\mathbf I; r_0, \alpha}}
{r_0!} \, %
\sqrt{n |\mathbf I| + 1} \cdot |\mathbf I|^{r_0 + \alpha}
\end{equation}
(recall %
\eqref{eq:Holder} for $|f|_{\mathbf I; r_0, \alpha}$) where $\bkt{n \mathbf I} 
\stackrel{{\rm def.}}{=} n \mathbf I \cap \{1, 2, \ldots\} \, (\subset [n])$ and we 
always interpret $0^{r_0 + \alpha} = 0$.
\end{lemma}
\begin{proof}
Since $\Pi^{(r)}_I$ is the orthogonal projector onto the subspace 
$S_I^{(r)}$ spanned by all polynomial vectors in $\R^I$ with degree $r$ 
(here $I \subset [n]$), it follows that
\begin{equation}\label{eq:invar}
({\rm Id} - \Pi^{(r)}_{I})\,\theta_{I} = ({\rm Id} - 
\Pi^{(r)}_{I})\,\theta'_I
\end{equation}
for any $\theta' \in \R^n$ satisfying $(\theta - \theta')_I \in 
S_I^{(r)}$ where ${\rm Id}$ is the identity operator on $\R^I$. Consequently, in view of the definition of $Q^{(r)}(\theta; 
I_1, I)$ in \eqref{def:Q} (the first expression in particular), we have
\begin{equation}\label{eq:Qinvar}
Q^{(r)}(\theta; I_1, I) = Q^{(r)}(\theta'; I_1, I)    
\end{equation}
for any such $\theta$ and $\theta'$. %
Now since 
$$\Pi_{\bkt{n\mathbf I}}^{(r)}\theta^\ast_{\bkt{n\mathbf 
I}} \in %
S_{\bkt{n\mathbf I}}^{(r)},$$
there is a polynomial $p: [0, 1] \to \R$ 
of degree %
$r$ satisfying $(\Pi_{\bkt{n\mathbf 
I}}^{(r_0)}\theta^\ast_{\bkt{n\mathbf I}})_{i} = p(\frac i n)$ for all 
$i \in \bkt{n\mathbf I}$ so that, with $\bar \theta^\ast \stackrel{{\rm 
def.}}{=} ((f - p)(\frac i n))_{i \in [n]} \in \R^n$,
\begin{equation*}
\Pi_{\bkt{n\mathbf I}}^{(r)}\bar \theta^\ast_{\bkt{n\mathbf I}} = 0.
\end{equation*}
Also, we have $|f - p|_{\mathbf I; r_0, \alpha} = |f|_{\mathbf I; r_0, 
\alpha}$ owing to its definition in \eqref{eq:Holder} as $p$ is a 
degree %
$r$ polynomial and $r_0 \le r$. Therefore, in view of \eqref{eq:Qinvar}, 
\eqref{eq:Tr_bnd_Holder} amounts to the same statement with 
$\theta^\ast$ replaced by $\bar \theta^\ast$. In other words, %
we can assume without any loss of generality that 
\begin{equation}\label{eq:PiIvanish}
\Pi_{\bkt{n\mathbf I}}^{(r)} %
\theta^\ast_{\bkt{n\mathbf I}} = 0.
\end{equation}
Now %
by \eqref{eq:Qformmat}, %
we can write
\begin{equation}\label{eq:Qr_upper_bnd}
Q^{(r)}(%
\theta^\ast, I_1, \bkt{n \mathbf I}) = %
\|\Pi_{I_1, I_2}^{(r)} \theta^\ast_{\bkt{n \mathbf I}}\|^2 \le \big\|\Pi_{I_1, 
I_2}^{(r)}\big\|^2 \,\, %
\|\theta^\ast_{\bkt{n \mathbf I}}\|^2 \le %
\|\theta^\ast_{\bkt{n \mathbf I}}\|^2
\end{equation}
for any %
interval $I_1 \subset \bkt{n \mathbf I}$ %
where, in the final step, we used that 
$\big\|\Pi_{I_1, I_2}^{(r)}\big\| \le 1$ as it is an orthogonal 
projection. Now, letting $\bkt{n \mathbf I} = \bkt{a, b}$ where $a, b 
\in [n]$, consider the vector %
${\rm Tayl}_{\bkt{n \mathbf I}}^{(r_0)}(f) \in S^{(r)}_{\bkt{n \mathbf 
I}}(\supset S^{(r_0)}_{\bkt{n \mathbf I}})$ defined as
\begin{equation}\label{def:Taylor}
({\rm Tayl}_{\bkt{n \mathbf I}}^{(r_0)}(f))_i = \sum_{0 \le k \le r_0} 
\frac{f^{(k)}(\tfrac a n)}{k!} \, \frac{(i - a)^k}{n^k} \text{ for all 
$i \in I$.}
\end{equation}
Since $f \in C^{r_0, \alpha}(\mathbf I)$, it follows from Taylor's theorem that
\begin{equation}\label{eq:taylor_approx}
\|
\theta^\ast_{\bkt{n \mathbf I}} - {\rm Tayl}_{\bkt{n \mathbf I}}^{(r_0)}
(f)\|_{\infty} \le \frac{|f|_{\mathbf I; r_0, \alpha}}{r_0!} \, |\mathbf I|^{r_0 
+ \alpha}
\end{equation}
where $\|\eta\|_\infty \stackrel{{\rm def.}}{=} \max_{j \in J}|\eta_j|$ denotes the $\ell_\infty$-norm for any $\eta \in \R^J$ ($J \subset [n]$) and hence
\begin{equation*}
\begin{split}
\|\theta^\ast_{\bkt{n \mathbf I}}\| &\stackrel{\eqref{eq:PiIvanish}}{=} \|\theta^\ast_{\bkt{n \mathbf I}} - \Pi_{\bkt{n \mathbf I}}^{(r_0)} \theta^\ast_{\bkt{n \mathbf I}}\| \le \|\theta^\ast_{\bkt{n \mathbf I}} - {\rm Tayl}_{\bkt{n \mathbf I}}^{(r_0)}(f)\| \\
&\le \|\theta^\ast_{\bkt{n \mathbf I}} - {\rm Tayl}_{\bkt{n \mathbf I}}^{(r_0)}(f)\|_\infty \, \sqrt{n |\mathbf I| + 1}  \le 
\frac{|f|_{\mathbf I; r_0, \alpha}}{r!} \, |\mathbf I|^{r_0 + \alpha} \, \sqrt{n |\mathbf I| + 1}.
\end{split}
\end{equation*}
Plugging this into \eqref{eq:Qr_upper_bnd} and taking maximum over all 
partitions $\{I_1, I_2\}$ of $\bkt{n \mathbf I}$
with $I_1$ an interval, we can deduce \eqref{eq:Tr_bnd_Holder} in view 
of the definition of $T^{(r)}_{\theta}(\bkt{n \mathbf I})$ in \eqref{def:discrep}.
\end{proof}

Now we are ready to state our bound on the variance term.

\begin{proposition}\label{prop:variance}
Under the assumptions of Theorem~\ref{thm:main}, the following bound holds:
\begin{equation}\label{eq:variance_bnd}
|{\rm N}_{i_0}| \le C(r) \, 
\big(\sigma^{\frac{2\alpha}{2\alpha + 
1}}|f|_{\mathbf I_0; r_0, 
\alpha_0}^{\frac1{2\alpha + 1}}\, (\log n)^{\frac{\alpha}{2\alpha + 1}}n^{-\frac{\alpha}
{2\alpha + 1}} + \sigma\sqrt{\log n}\,(ns_0)^{-\frac12})
\end{equation}
w.h.p. simultaneously for all %
quadruplets $(i_0, s_0, r_0, \alpha_0)$ satisfying $f \in C^{r_0, 
\alpha_0}(\mathbf I_0)$ where $\alpha = \alpha_0 + r_0$.
\end{proposition}
\begin{proof}%
Let us first cover the cases where either Lemma~\ref{lem:Tr_bnd_Holder} or 
\eqref{eq:noise_threshold_decomp_lb} does {\em not} apply, i.e., if $\bkt{i_0 \pm \hat{h}_{i_0}} = [n]$ or if $\bkt{i_0 \pm (\hat h_{i_0} + 1)} \not\subset \bkt{n\mathbf I_0}$ where ${\mathbf I}_0 = 
[\tfrac{i_0}n \pm s_0]$ (see below \eqref{eq:main}). Then clearly $\hat l_{i_0} \ge 
c n s_0$ and hence, in view of \eqref{Piep_bnd},
\begin{equation}\label{eq:easy_bnd}
|{\rm N}_{i_0}| \le C(r) \sigma \sqrt{\frac{\log n}{n s_0}}.
\end{equation}

So let us assume that $\bkt{i_0 \pm \hat{h}_{i_0}} \ne [n]$ and also $\bkt{i_0 \pm 
(\hat h_{i_0} + 1)} \subset \bkt{n\mathbf I_0}$. By \eqref{eq:noise_threshold_decomp_lb} in 
Proposition~\ref{prop:noise_threshold_decomp}, we have
\begin{equation*}
T^{(r)}_{\theta^\ast}(\bkt{i_0 \pm (\hat h_{i_0} + 1)}) \ge \lambda - 
\Cr{C:noise} \sigma \sqrt{\log n}
\end{equation*}
w.h.p. simultaneously for all $i_0 \in [n]$. Also since %
$\bkt{i_0 \pm (\hat h_{i_0} + 1)} \subset \bkt{n\mathbf I_0}$ in this case, %
Lemma~\ref{lem:Tr_bnd_Holder} yields us that if $f \in C^{r_0, \alpha_0}
(\mathbf I_0)$, then
\begin{equation*}
T^{(r)}_{\theta^\ast}(\bkt{i_0 \pm (\hat h_{i_0} + 1)}) \le 
\frac{|f|_{\mathbf I_0; r_0, \alpha_0}}{n^{r_0 + \alpha_0}} \, %
(\hat l_{i_0} + 2)^{r_0 + \alpha_0 + \frac 12}.
\end{equation*}
Together the last two displays imply, when %
\begin{equation}\label{choice:lambda}
\lambda = 2\Cr{C:noise} \sigma \sqrt{\log n},    
\end{equation}
that
\begin{equation}\label{eq:bandwidthlb}
\hat l_{i_0} + 2 \ge 
c\, (%
\sigma \sqrt{\log n})^{\frac1{r_0 + 
\alpha_0 + \frac12}} \, |f|_{\mathbf I_0; r_0, \alpha_0}^{-\frac1{r_0 + 
\alpha_0 + \frac12}}\, n^{\frac{r_0 + \alpha_0}{r_0 + \alpha_0 + \frac12}} 
\:\: 
\text{w.h.p.}
\end{equation}
simultaneously for all %
quadruplets $(i_0, s_0, r_0, \alpha_0)$ satisfying $f \in 
C^{r_0, \alpha_0}(\mathbf I_0)$, $\bkt{i_0 \pm (\hat h_{i_0} + 1)} \subset 
\bkt{n\mathbf I_0}$ and $\bkt{i_0 \pm \hat h_{i_0}} \ne [n]$. Plugging this into \eqref{Piep_bnd} and combining with \eqref{eq:easy_bnd}, we 
obtain~\eqref{eq:variance_bnd}.
\end{proof}

\noindent {\bf Stage $4$: Bounding the bias term.}
We now show how the other half of the bandwdith selection equation, i.e., \eqref{eq:noise_threshold_decomp_ub} leads to a bound on the bias term. 
\begin{proposition}\label{prop:bias}
Under the assumptions of Theorem~\ref{thm:main}, with $\alpha = 
\alpha_0 + r_0$, 
\begin{equation}\label{eq:biasfinalbd}
|{\rm B}_{i_0}| \le C(r) \, \big(\sigma^{\frac{2\alpha}{2\alpha + 1}} 
|f|_{\mathbf I_0; r_0, \alpha_0}^{\frac1{2\alpha + 1}}\, (\log 
n)^{\frac{\alpha}{2\alpha + 1}}n^{-\frac{\alpha}{2\alpha + 1}} + 
\sigma \sqrt{\log n}\,(ns_0)^{-\frac12})
\end{equation}
w.h.p. simultaneously for all triplets $(i_0, s_0, r_0, \alpha_0)$ satisfying 
$f \in C^{r_0, \alpha_0}(\mathbf I_0)$. 
\end{proposition}
The proof of Proposition~\ref{prop:bias} takes a bit of work. In order to analyze 
the bias term ${\rm B}_{i_0}$ from~\eqref{eq:bias_noise}, we will further decompose 
it as follows. For any %
interval $I_1 \subset \bkt{i_0 \pm \hat{h}_{i_0}}$ such that $i_0 \in I_1$, we can 
write
\begin{equation}\label{eq:bias_two_parts}
\begin{split}
{\rm B}_{i_0} & = \big(\proj^{(r)}_{\bkt{i_0 \pm \hat 
h_{i_0}}} \theta^\ast_{\bkt{i_0 \pm \hat{h}_{i_0}}}\big)_{i_0} - 
\theta_{i_0}^\ast \\
&= \underbrace{\big(\proj^{(r)}_{\bkt{i_0 \pm \hat 
h_{i_0}}} \theta^\ast_{\bkt{i_0 \pm \hat{h}_{i_0}}}\big)_{i_0} - 
\big(\proj^{(r)}_{I_1} \theta^\ast_{I_1}\big)_{i_0}}_{{\rm B}_{i_0,1}} + 
\underbrace{\big(\proj^{(r)}_{I_1} \theta^\ast_{I_1}\big)_{i_0} - \theta_{i_0}^\ast}_{{\rm B}_{i_0,2}}.
\end{split}
\end{equation}
The intuition behind decomposing the bias term as above is the following. So far, 
under the assumption that $\theta^*$ is locally H\"{o}lder at the location $i_0$, we 
have obtained a lower bound on the length $\hat{l}_{i_0}$ as 
in~\eqref{eq:bandwidthlb}. If we had a matching upper bound, then we could have 
directly bounded ${\rm B}_{i_0}$. However, the local H\"{o}lder smoothness does not 
preclude the signal being even {\em smoother}, i.e., the \textit{true} H\"{o}lder 
exponent may very well be larger than $\alpha_0$. In such a case, one would expect 
the length $\hat{l}_{i_0}$ to be even larger. The above decomposition identifies 
this case where the second term ${\rm B}_{i_0,2}$ corresponds to the ideal bandwidth 
case and the first term ${\rm B}_{i_0,1}$ accounts for the potentially extra bias 
arising out of extra smoothness. We will see in Lemma~\ref{lem:T2} below that ${\rm B}_{i_0,2}$ can be bounded using the H\"{o}lder smoothness condition while in 
Lemma~\ref{lem:T1}, ${\rm B}_{i_0,1}$ will be shown to be of order at most 
$T^{(r)}_{\theta^\ast}(\bkt{i_0 \pm \hat h_{i_0}}) / \sqrt{|I_1|}$. 
\begin{lemma}\label{lem:T2}
Under the same set-up as in Lemma~\ref{lem:Tr_bnd_Holder}, we have
\begin{equation}\label{eq:T2}
\big\|\big({\rm Id} - \Pi_{\bkt{n \mathbf I}}^{(r)}\big) \, \theta^\ast_{\bkt{n \mathbf I}} \big\|_{\infty} 
\le %
C(r) |f|_{\mathbf I; r_0, \alpha} |\mathbf I|^{r_0 + \alpha}.
\end{equation}
\end{lemma}
\begin{proof}
Since ${\rm Tayl}_{\bkt{n \mathbf I}}^{(r_0)}(f) \in S_{\bkt{n \mathbf I}}^{(r)} (\supset S_{\bkt{n \mathbf I}}^{(r_0)})$ (see \eqref{def:Taylor}) and $\Pi^{(r)}_{\bkt{n \mathbf I}}$ is the orthogonal projector 
onto $S_{\bkt{n \mathbf I}}^{(r)}$, we have
\begin{equation*}
\big({\rm Id} - \Pi_{\bkt{n \mathbf I}}^{(r)}\big) \, \theta^\ast_{\bkt{n \mathbf I}} = \big({\rm Id} - 
\Pi_{\bkt{n \mathbf I}}^{(r)}\big) \, \big(\theta^\ast_{\bkt{n \mathbf I}} - {\rm Tayl}_{\bkt{n \mathbf I}}^{(r_0)}
(f)\big).
\end{equation*}
Therefore, 
\begin{equation*}
\begin{split}
\big\|\big({\rm Id} - \Pi_{\bkt{n \mathbf I}}^{(r)}\big) \, \theta^\ast_{\bkt{n \mathbf I}}\big\|_{\infty} &=  \big\|\big({\rm Id} - \Pi_{\bkt{n \mathbf I}}^{(r)}\big) \, \big(\theta^\ast_{\bkt{n \mathbf I}} - {\rm 
Tayl}_{\bkt{n \mathbf I}}^{(r_0)}(f)\big)\big\|_\infty\\
& \le \big\|\big({\rm Id} - \Pi_{\bkt{n \mathbf I}}^{(r)}\big) \big\|_{\infty} \, 
\big\| \big(\theta^\ast_{\bkt{n \mathbf I}} - {\rm Tayl}_{\bkt{n \mathbf I}}^{(r_0)}(f)\big) \big\|_{\infty}\\
& \stackrel{\eqref{eq:projector2} + \eqref{eq:taylor_approx}}{\le} %
C(r)\, %
|f|_{\mathbf I; r_0, \alpha} |\mathbf I|^{r_0 + \alpha}. \qedhere
\end{split}
\end{equation*}
\end{proof}

\begin{lemma}\label{lem:T1}
Let $I$ be a sub-interval of $[n]$ and $\lambda \ge 0$. Then for any 
interval $I_1 \subset I$ %
and $\theta \in \R^n$, we have
\begin{equation}\label{eq:T1}
\|(\Pi^{(r)}_{I} \theta_I)_{I_1} -  \Pi^{(r)}_{I_1}\theta_{I_1}\|_{\infty} 
\le C(r) \, \frac {T^{(r)}_{\theta}(I)}{\sqrt{|I_1|}}.
\end{equation}
\end{lemma}
\begin{proof}
Using the same invariance argument as in the proof of 
Lemma~\ref{lem:Tr_bnd_Holder}, in particular %
the display \eqref{eq:invar}, we can assume without any loss of 
generality that
\begin{equation}\label{eq:PiIvanish1}
\Pi_I^{(r)} \theta_I = 0.
\end{equation}
But in that case we can write
\begin{equation}\label{eq:Ivanish}
\|(\Pi^{(r)}_{I} \theta_I)_{I_1} -  \Pi^{(r)}_{I_1}\theta_{I_1}\|_{\infty} 
= \|\Pi^{(r)}_{I_1}\theta_{I_1}\|_{\infty}.
\end{equation}
Also, %
\begin{equation}\label{eq:PiI_1Tr}
\begin{split}
\|\Pi_{I_1}^{(r)}\theta_{I_1}\|^2 
\stackrel{\eqref{def:Q}+\eqref{eq:PiIvanish1}}{\le} Q^{(r)}(\theta; I_1, 
I) \stackrel{\eqref{def:discrep}}{\le} (T^{(r)}_{\theta}(I))^2.
\end{split}
\end{equation}
Since $\Pi_{I_1}^{(r)}\theta_{I_1} \in S_{I_1}^{(r)}$, it follows from 
Lemma~\ref{lem:projector} in the next subsection that
\begin{equation*}
\|\Pi_{I_1}^{(r)}\theta_{I_1}\|_{\infty} \le C(r) 
\, \frac{\|\Pi_{I_1}^{(r)}\theta_{I_1}\|}{\sqrt{|I_1|}}.
\end{equation*}
Combined with \eqref{eq:Ivanish} and \eqref{eq:PiI_1Tr}, this yields 
\eqref{eq:T1}.
\end{proof}

We are now ready to prove Proposition~\ref{prop:bias}.

\begin{proof}[Proof of Proposition~\ref{prop:bias}]
Recalling the two parts of the bias parts ${\rm B}_{i_0,1}$ and ${\rm B}_{i_0,2}$ 
from \eqref{eq:bias_two_parts}, we now bound them separately. Firstly, we can write
\begin{align*}\label{eq:biasbnd1}
   |{\rm B}_{i_0,1}| & \,\, = |\big(\proj^{(r)}_{\bkt{i_0 \pm \hat 
h_{i_0}}} \theta^\ast_{\bkt{i_0 \pm \hat{h}_{i_0}}}\big)_{i_0} - 
\big(\proj^{(r)}_{I_1} \theta^\ast_{I_1}\big)_{i_0}| \\
& \,\, \le \big\|\proj^{(r)}_{[i_0 \pm \hat 
h_{i_0}]} \theta^\ast_{[i_0 \pm \hat{h}_{i_0}]} - 
\proj^{(r)}_{I_1} \theta^\ast_{I_1}\big\|_{\infty} \\
&\stackrel{\eqref{eq:T1}}{\le} C(r) \, \frac{T^{(r)}_{\theta^\ast}
(\bkt{i_0 \pm \hat{h}_{i_0}})}{\sqrt{|I_1|}} \\ 
&\,\,\stackrel{\eqref{eq:noise_threshold_decomp_ub}}{\le} C(r) \sigma \frac{\sqrt{\log n}}{\sqrt{|I_1|}}.
\end{align*}
On the other hand,
\begin{align*}%
   |{\rm B}_{i_0,2}| = |\big(\proj^{(r)}_{I_1} \theta^\ast_{I_1}\big)_{i_0} - \theta_{i_0}^\ast| \le \|({\rm Id} - 
\Pi_{I_1}^{(r)})\theta_{I_1}^\ast\|_{\infty} \stackrel{\eqref{eq:T2}}{\le} C(r) |f|_{\mathbf I_0; r, \alpha}\big(\tfrac{|I_1| - 
1}n\big)^{r_0 + \alpha_0},
\end{align*}
where for the second inequality we also need $I_1 \subset \bkt{n\mathbf I_0}$. Now 
setting $I_1 = \bkt{i_0 \pm \tilde s_1}$ where
\begin{equation*}
\tilde s_1  = c(r) \, (\sigma\sqrt{\log n})^{\frac1{r_0 + 
\alpha_0 + \frac12}} \, |f|_{\mathbf I_0; r_0, \alpha_0}^{-\frac1{r_0 + 
\alpha_0 + \frac12}}\, n^{\frac{r_0 + \alpha_0}{r_0 + \alpha_0 + 
\frac12}}  \wedge %
ns_0,
\end{equation*}
we see in view of \eqref{eq:bandwidthlb} that we can make $I_1 \subset 
\bkt{i_0 \pm \hat h_{i_0}} \cap \bkt{n\mathbf I_0}$ by suitably choosing $c(r)$. 
Therefore, we can plug the value of $|I_1|$ corresponding to this choice 
into the bounds on ${\rm B}_{i_0, 1}$ and ${\rm B}_{i_0, 2}$ obtained above and 
add them up to finally obtain~\eqref{eq:biasfinalbd} in view of 
\eqref{eq:bias_two_parts}.
\end{proof}

\begin{proof}[Putting everything together and the proof of Theorem~\ref{thm:main}]
Combining the bound on bias term as in Proposition~\ref{prop:bias} with the one on 
the variance term given by Proposition~\ref{prop:variance}, we deduce 
\eqref{eq:main}.
\end{proof}

\subsection{Some properties of polynomials}

We needed a result bounding the diagonal entries of the projection matrix for the subspace of polynomials. The following result is stated and proved in~\cite{chatterjee2024minmax} (see Proposition $13.1$ therein).

\begin{lemma}\label{lem:poly1}
Fix an integer $r \geq 0$. For any positive integers $1 \leq m \leq n$, for any interval $I \subset [n]$ with $|I| = m$, there exists a constant $C_r > 0$ only depending on $r$ such that 
\begin{equation}\label{eq:poly1}
\max_{i \in I} \, (\Pi_I^{(r)})_{i, i} \leq \frac{C_r}{m}.
\end{equation}
\end{lemma}

We also used the following result which is a particular property of a discrete polynomial vector. It says that the average $\ell_2$-norm is comparable to $\ell_{\infty}$-norm for any vector which is a discrete polynomial.

\begin{lemma}\label{lem:projector}
For any (non-empty) sub-interval $I$ of $[n]$ and $\eta \in S_{I}^{(r)}$ 
(recall \eqref{def:SIr}), we have
\begin{equation}\label{eq:projector1}
\|\eta\|_{\infty} \le C(r) \, \frac{\|\eta\|}{\sqrt{|I|}}.
\end{equation}
Furthermore, letting $\|\Pi_{I}^{(r)}\|_{\infty} \stackrel{{\rm def.}}
{=} \sup_{\theta \ne 0} \frac{\|\Pi_{I}^{(r)}\theta\|_{\infty}}{\|\theta\|_{\infty}}$ denote the operator norm of 
$\Pi_{I}^{(r)}$ w.r.t. to the $\ell_\infty$-norm on $\R^I$, we have
\begin{equation}\label{eq:projector2}
\|\Pi_{I}^{(r)}\|_{\infty} \le C(r).
\end{equation}
\end{lemma}
\begin{proof}
Since $I$ is an interval, we can use Lemma~13.1 and 13.2 in 
\cite{chatterjee2024minmax} to deduce the existence of an orthonormal basis (ONB) $\{\tilde \eta^{(k)}\}_{0 \le k 
\le r}$ for $S_I^{(r)}$ such that
\begin{equation}\label{eq:l2linfty}
\max_{0 \le k \le r \wedge (|I|-1)}  \|\tilde \eta^{(k)}\|_{\infty} \le 
\frac{C(r)}{\sqrt{|I|}}.
\end{equation}
Now writing for any $\eta \in S_I^{(r)}$,
\begin{equation*}
\eta = \sum_{0 \le k \le r \wedge (|I| - 1)} a_k\, \tilde \eta^{(k)}
\end{equation*}
with $a_k \in \R$, whence
\begin{equation*}
\|\eta\|_2 = \big(\sum_{0 \le k \le r \wedge (|I| - 1)} 
a_k^2\big)^{\frac12} \text{ and } \|\eta\|_{\infty} \le \max_{0 \le k \le 
r \wedge (|I|-1)} \|\tilde \eta^{(k)}\|_{\infty} \: \big(\sum_{0 
\le k \le r \wedge (|I|-1)} |a_k|\big).
\end{equation*}
By the Cauchy-Schwarz inequality, 
\begin{equation*}
\sum_{0 
\le k \le r \wedge (|I|-1)} |a_k| \le \sqrt{r}\, \big(\sum_{0 \le k \le r \wedge (|I| 
- 1)} a_k^2\big)^{\frac12} = \|\eta\|_2.
\end{equation*}
Combined with \eqref{eq:l2linfty}, the last two displays yield 
\eqref{eq:projector1}.

Next we prove \eqref{eq:projector2}. Identifying $\Pi_I^{(r)}$ with the corresponding matrix, we have the following standard expression for $\|\Pi_I^{(r)}\|_{\infty}$.
\begin{equation*}
\|\Pi_I^{(r)}\|_{\infty} = \max_{i \in I} \, \sum_{j \in I} \, |
(\Pi_I^{(r)})_{i, j}|.
\end{equation*}
Since $\Pi_I^{(r)}$ is an orthogonal projection, it is idempotent and 
symmetric (as a matrix) and hence
\begin{equation*}
\sum_{j \in I} \, ((\Pi_I^{(r)})_{i, j})^2 = (\Pi_I^{(r)})_{i, i}
\end{equation*}
for each $i \in I$. Consequently by the Cauchy-Schwarz inequality,
\begin{equation}\label{eq:inftynormbnd}
\|\Pi_I^{(r)}\|_{\infty} = \max_{i \in I} \, \sum_{j \in I} \, |(\Pi_I^{(r)})_{i, 
j}| \le \sqrt{|I|} 
\, \sqrt{\max_{i \in I} \, \sum_{j \in I} \, ((\Pi_I^{(r)})_{i, j})^2} = 
\sqrt{|I|} \sqrt{\max_{i \in I} \, (\Pi_I^{(r)})_{i, i}} \, .
\end{equation}
Plugging \eqref{eq:poly1} into the right-hand side of above, we obtain 
\eqref{eq:projector2}.
\end{proof}

\section{Algorithm and simulations}\label{sec:algsimu}
In this section we first discuss the computational aspects of \laser{}. After that, we provide a comparative study of our method vis-\`{a}-vis other popular methods in the literature backed by simulation studies.
\subsection{Pseudocode for \laser{}}\label{sec:pseudo}
We now %
present %
a pseudocode for \laser{}. For ease of understanding, we have broken down the full algorithm into three subroutines. Algorithm~\ref{alg:compute_discrepancy}, named \texttt{ComputeDiscrepancy}, outputs for an interval $I$, the discrepancy criterion \eqref{def:Q}. 
Algorithm \ref{alg:bandwidth_selector}, called \texttt{BandwidthSelector}, \sloppy uses \texttt{ComputeDiscrepancy} to compute local bandwidths \'{a} la \eqref{def:bandwidth} at co-ordinates of interest given the tuning parameter $\lambda$. Algorithm~\ref{alg:laser} calls \texttt{BandwidthSelector} at each co-ordinate to obtain a local bandwidth and performs a local polynomial regression to output the final estimate at that co-ordinate. We note that the degree of polynomial regression $r$ is an input parameter, which can be specified by the user.

\begin{algorithm}[H]
\caption{\textsf{ComputeDiscrepancy}: Compute the discrepancy criterion}\label{alg:compute_discrepancy}
\begin{algorithmic}[1]
\Require Interval $I$; vector $\theta$; degree parameter $r$, tuning parameter $\lambda$.
\Ensure Discrepancy criterion over $I$.
\For {all sub-intervals $I_1 \subset I$}
    \State Compute $Q^{(r)}(y, I_1, I)$.
\EndFor
\State  Return $\max_{I_1} Q^{(r)}(y, I_1, I)$.
\end{algorithmic}
\end{algorithm}

\begin{algorithm}[H]
\caption{\textsf{BandwidthSelector}: Select local bandwidth at a specified co-ordinate}\label{alg:bandwidth_selector}
\begin{algorithmic}[1]
\Require Requested co-ordinate $i_0$; data $y$; degree parameter $r$, tuning parameter $\lambda$.
\Ensure $\hat{h}_{i_0}$, local bandwidth.
\State Criterion $\gets 0$.
\State $h \gets -1$.
\While {Criterion $\le \lambda$}
    \State $h \gets h + 1$
    \State Criterion $\gets \textsf{ComputeDiscrepancy}(\bkt{i_0 \pm h}, y, X, \lambda, r)$.
\EndWhile
\State Return $h - 1$.
\end{algorithmic}
\end{algorithm}

\begin{algorithm}[H]
\caption{\laser{}: Locally Adaptive Smoothing Estimator for Regression}\label{alg:laser}
\begin{algorithmic}[1]
\Require Data $y$; degree parameter $r$, tuning parameter $\lambda$.
\Ensure $\hat{\theta}$.
\For {$i_0 = 1, \ldots, n$}
    \State $\hat{h}_{i_0} \gets \textsf{BandwidthSelector}(i_0; y, X, r, \lambda)$.
    \State Obtain $\hat{\theta}_{i_0}$ by fitting a degree-$r$ polynomial to $y_{\bkt{i_0 \pm \hat{h}_{i_0}}}$ using least squares.
\EndFor
\State Return $\hat{\theta}$.
\end{algorithmic}
\end{algorithm}

\subsubsection{Computational complexity and a fast dyadic variant}

First consider the complexity of \texttt{ComputeDiscrepancy}.
\begin{itemize}
    \item Each computation of $Q^{(r)}(y, I_1, I)$ in \texttt{ComputeDiscrepancy} involves a least squares fit with $O(|I|)$ observations and $r + 1$ variables. Thus each such computation incurs an $O(|I|)$ cost, where the hidden constant is dependent on $r$. Here and in the sequel, we hide such dependence on $r$ under the $O$-notation.
    \item There are $O(|I|^2)$ choices for the sub-intervals $I_1 \subset I$.
\end{itemize}
Thus the full complexity of \texttt{ComputeDiscrepancy} is $O(|I|^3)$. If one searches over intervals with endpoints and lengths on a dyadic scale, then the complexity of the second step reduces to $O((\log |I|)^2)$ and hence that of \texttt{ComputeDiscrepancy} to $O(|I|(\log |I|)^2)$.

It follows that the worst case complexity of a single run of \texttt{BandwidthSelector} is $\sum_{h \le n} O(h^3) = O(n^4)$. If, on the other hand, we search over $h$ on a dyadic scale as well and also use the dyadic version of \texttt{ComputeDiscrepancy}, then the complexity of \texttt{BandwidthSelector} becomes $O(n (\log n )^3)$.

It follows that the complexity of the full-blown variant of \laser{} is $O(n^5)$ whereas that of its dyadic variant is $O(n^2 (\log n )^3)$. In our reference implementation, we use these dyadic variants of \texttt{ComputeDiscrepancy} and \texttt{BandwidthSelector}. Our numerical experiments show that this dyadic variant has comparable performance to the full-blown version. In fact, although our theoretical risk bounds in Theorem~\ref{thm:main} are stated and proved for the full version of  our method, our proof in Section~\ref{sec:proof}, subject to minor modifications, yields similar bounds for this dyadic variant. Moreover, \laser{} naturally lends itself to a hierarchy of implementations with progressively coarser search spaces but \emph{lesser} computational tax, and an inspection of our proof reveals that the statistical performance of such variants degrade by \emph{at most} logarithmic factors.

\subsection{Numerical experiments}\label{sec:simu}
We compare \laser{} with three popular nonparametric regression methods, namely trend filtering (TF), wavelet thresholding (WT) and cubic smoothing splines (CSS) on the four test functions described in \cite{donoho1994ideal}. We have used the \texttt{genlasso}, \texttt{wavethresh} (with so-called \emph{universal tuning}) and \texttt{npreg} R packages to compute cross-validated versions of TF, WT and CSS, respectively. For \laser{}, we have developed an eponymous R package \texttt{laser}, available at \url{https://gitlab.com/soumendu041/laser}.

\begin{figure}[!th]
\centering
\includegraphics[scale = 0.45]{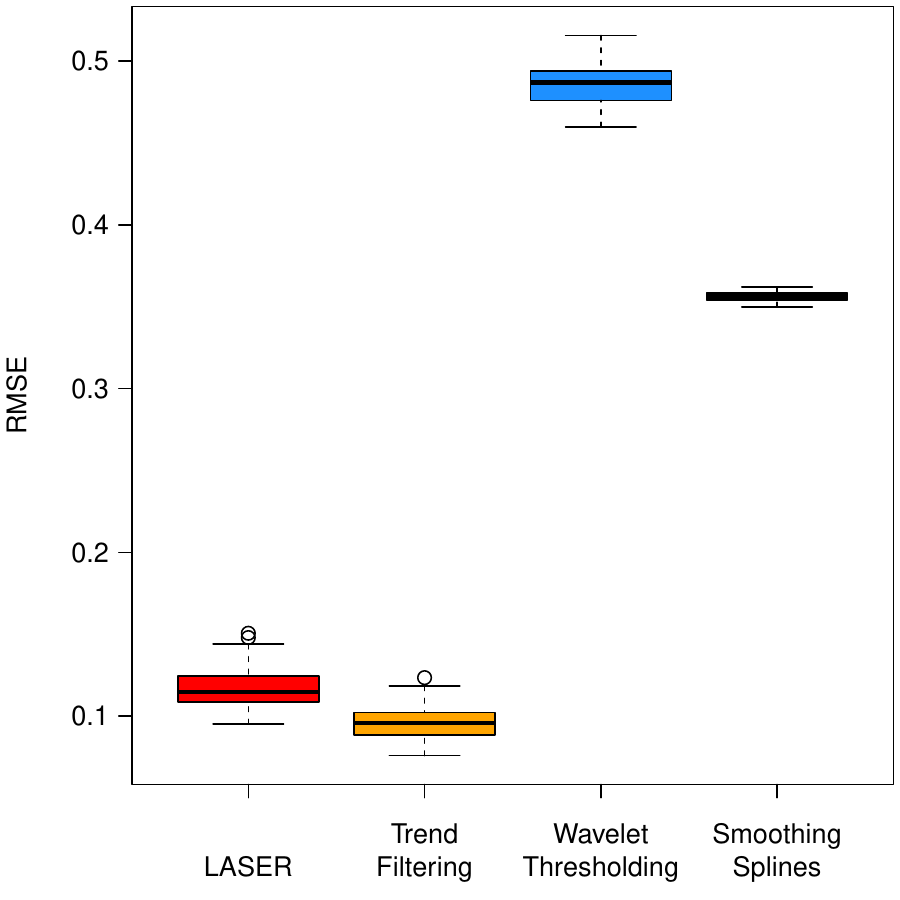}
\includegraphics[scale = 0.48]{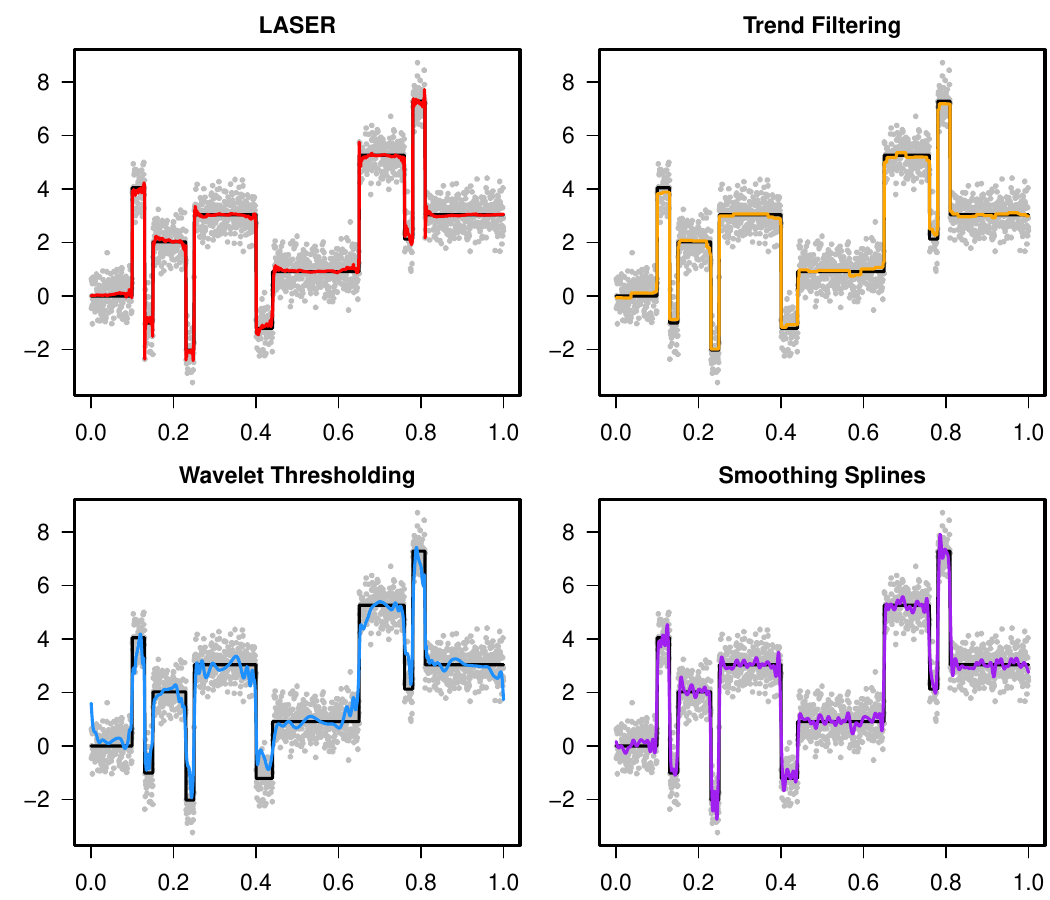}
\caption{The \texttt{Blocks} function. We have used \laser{} with $r = 0$ and $0$-th order Trend Filtering.} 
\label{fig:comp_blocks}
\end{figure}

\begin{figure}[!th]
\centering
\includegraphics[scale = 0.45]{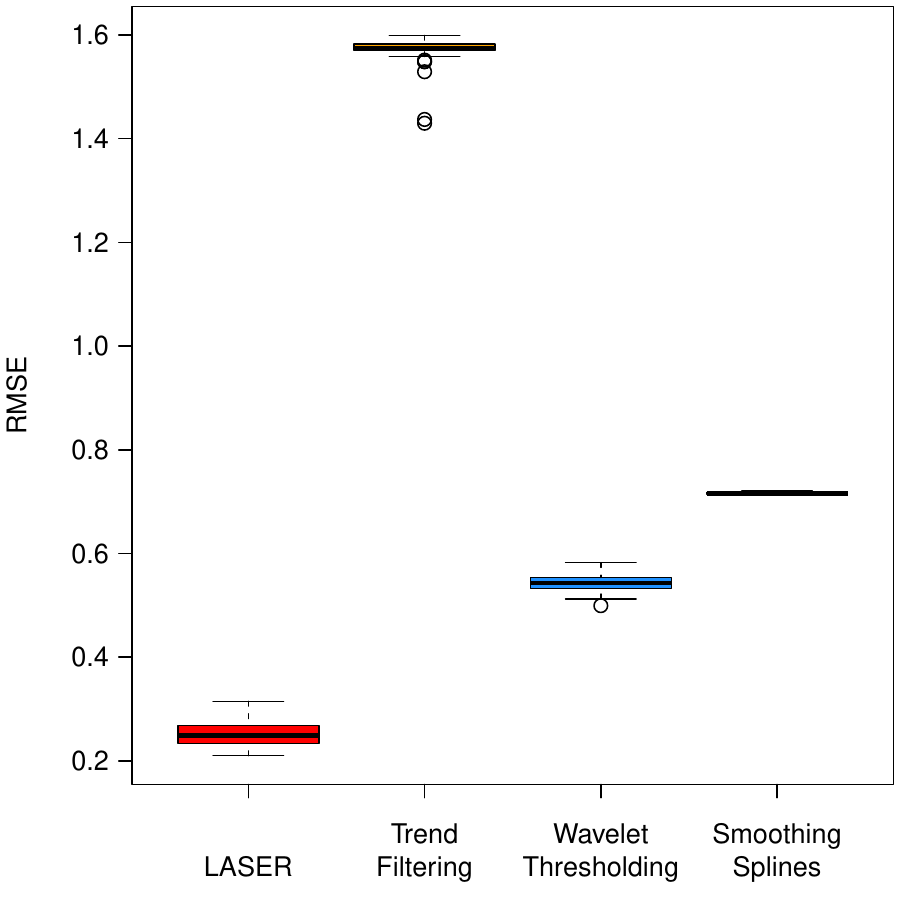}
\includegraphics[scale = 0.48]{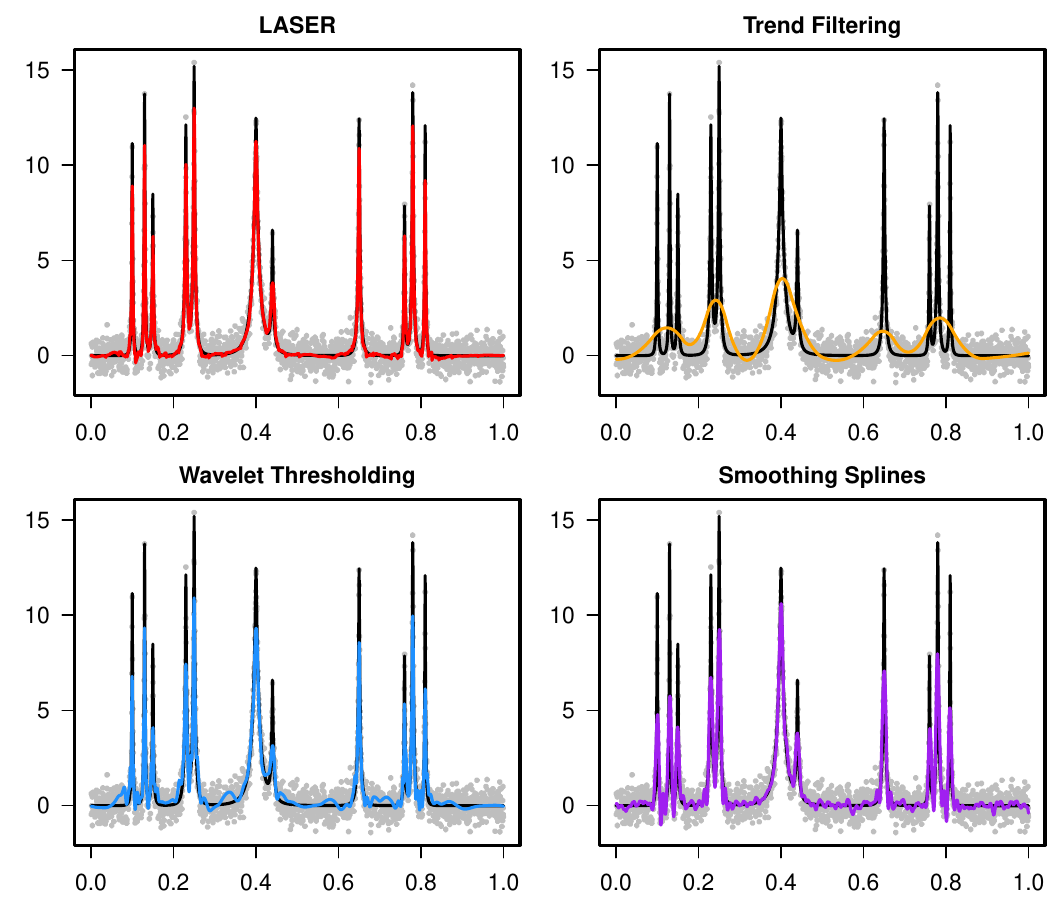}
\caption{The \texttt{Bumps} function. We have used \laser{} with $r = 2$ and $2$-nd order Trend Filtering.} 
\label{fig:comp_bumps}
\end{figure}

\begin{figure}[!th]
\centering
\includegraphics[scale = 0.45]{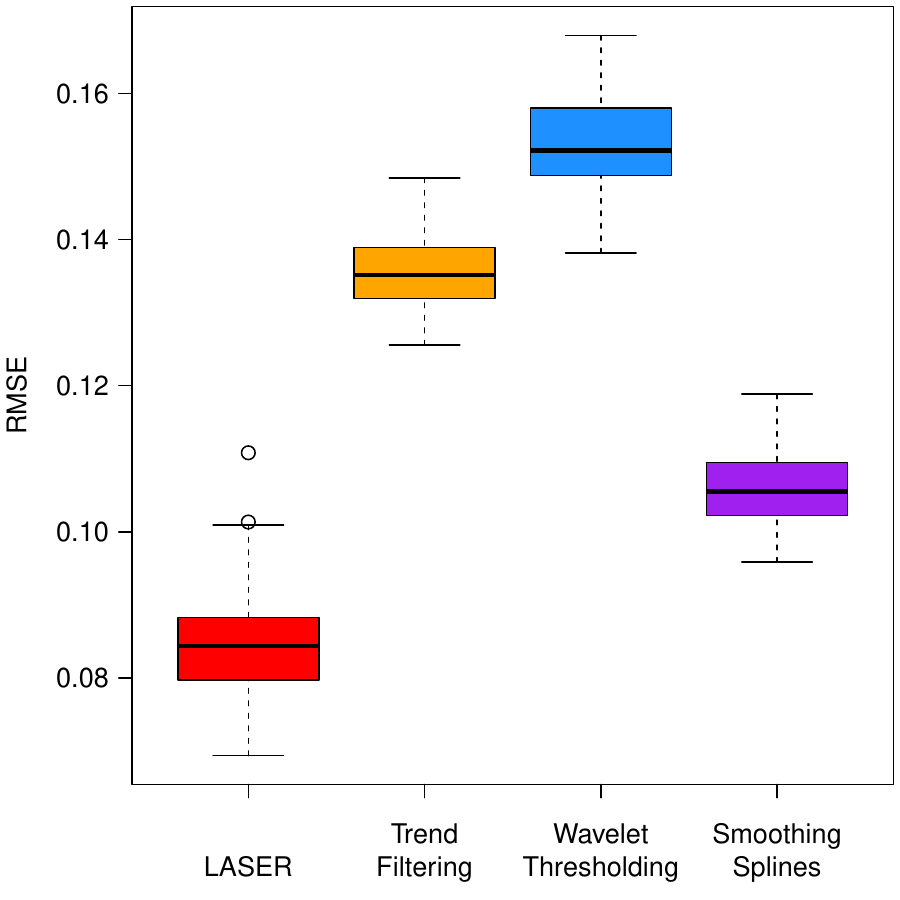}
\includegraphics[scale = 0.48]{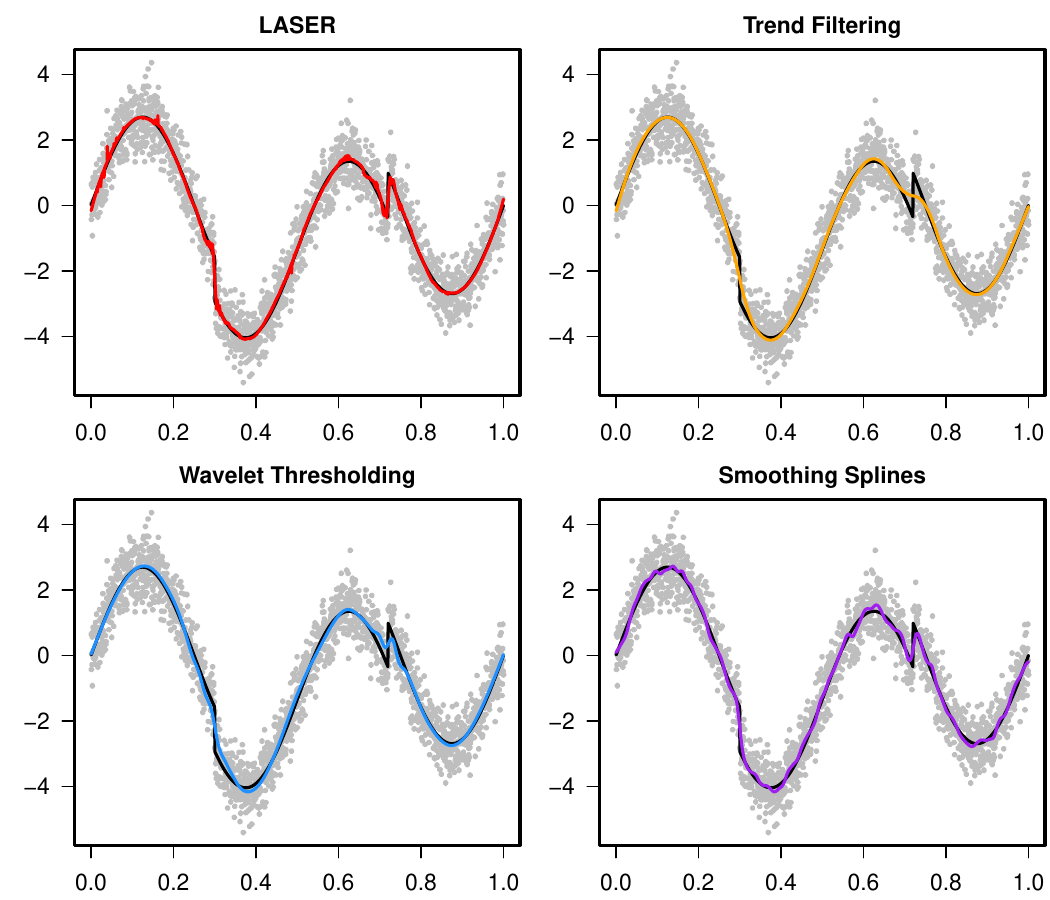}
\caption{The \texttt{HeaviSine} function. We have used \laser{} with $r = 2$ and $2$-nd order Trend Filtering.} 
\label{fig:comp_heavisine}
\end{figure}

\begin{figure}[!th]
\centering
\includegraphics[scale = 0.45]{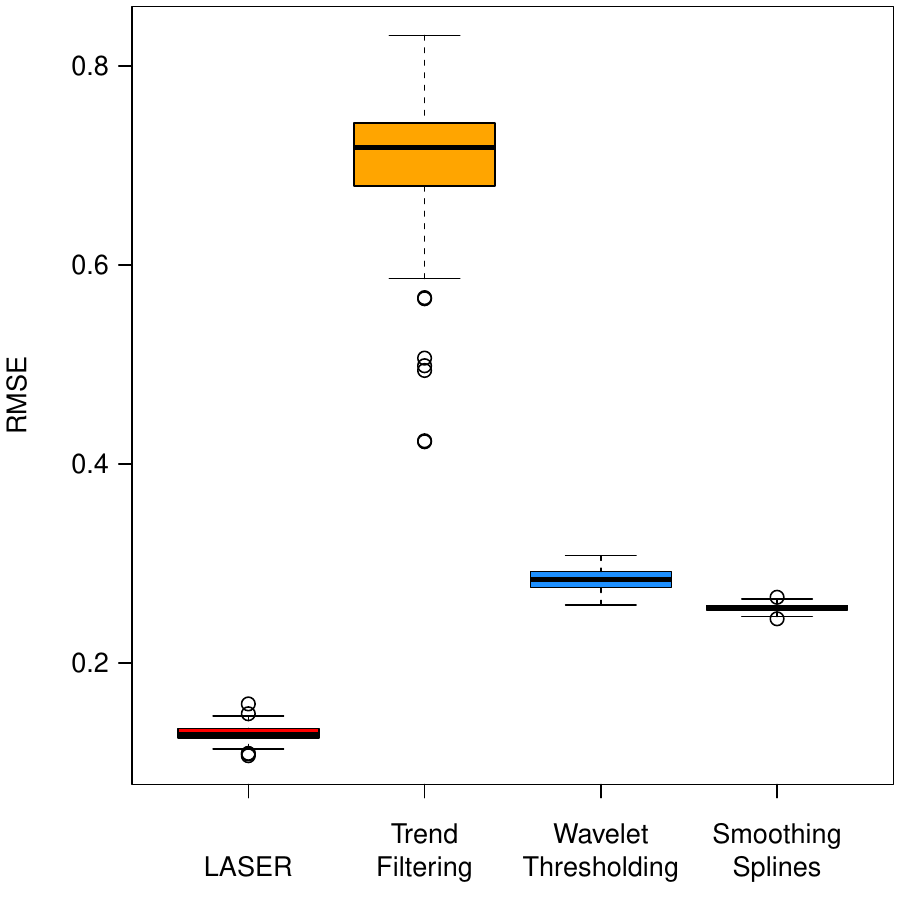}
\includegraphics[scale = 0.48]{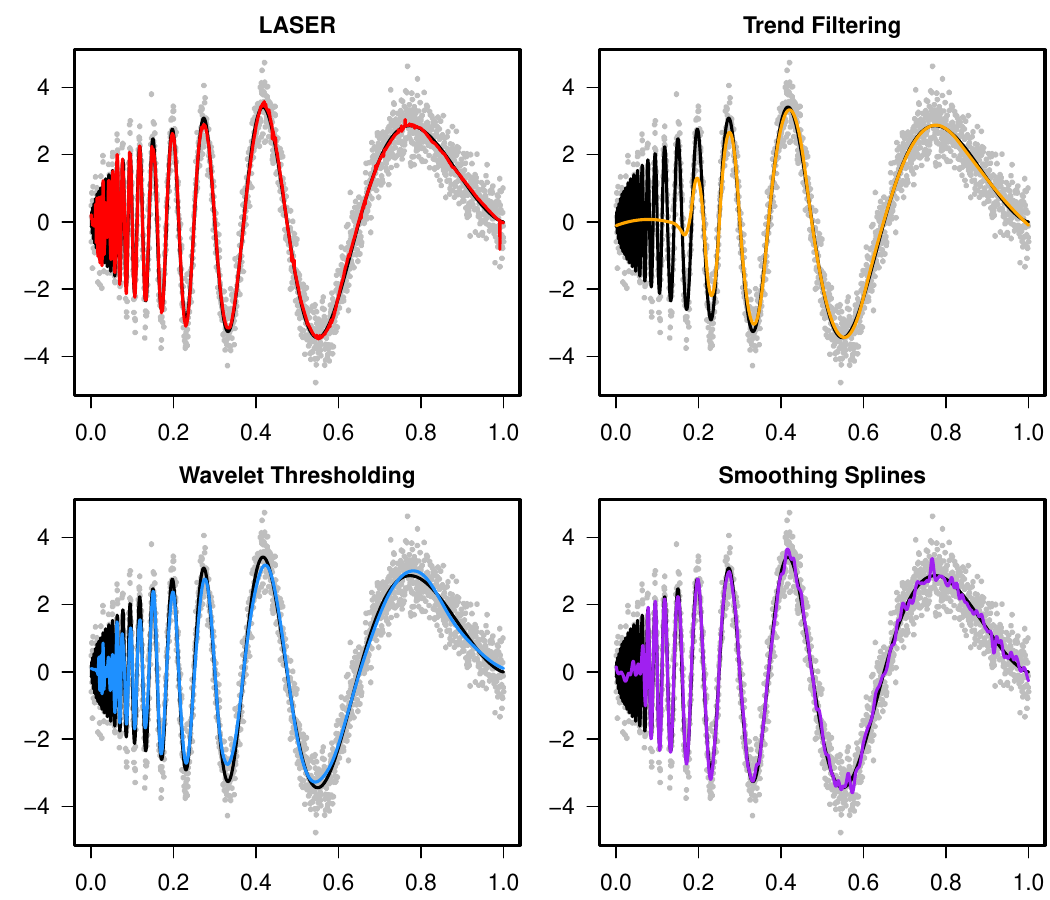}

\caption{The \texttt{Doppler} function. We have used \laser{} with $r = 2$ and $2$-nd order Trend Filtering.} 
\label{fig:comp_doppler}
\end{figure}

Figures~\ref{fig:comp_blocks}, \ref{fig:comp_bumps}, \ref{fig:comp_heavisine} and  \ref{fig:comp_doppler} show the results of four experiments, one for each of the functions \texttt{Blocks}, \texttt{Bumps}, \texttt{HeaviSine}, \texttt{Doppler} from \cite{donoho1994ideal}. The experimental set-up of each of these experiments, is as follows. For $f \in $ \{\texttt{Blocks}, \texttt{Bumps}, \texttt{HeaviSine}, \texttt{Doppler}\}, we set $\vartheta_f = (f(\tfrac{i}{n}))_{1 \le i \le n}$. The observations are generated as
\[
    y = \theta_f + \sigma \epsilon,
\]
where $\sigma > 0$, $\epsilon \sim N_n(0, \mathrm{Id})$ and
\[
    \theta_f := \mathrm{SNR} \cdot \sigma \cdot \frac{\vartheta_f}{\mathrm{sd}(\vartheta_f)}.
\]
Here $\mathrm{sd}(x) := \frac{1}{n}\sum_{i = 1}^n x_i^2 - (\frac{1}{n}\sum_{i = 1}^n x_i)^2$ denotes the numerical standard deviation of a vector $x \in \mathbb{R}^n$.
The factor $\mathrm{SNR}$ captures the signal-to-noise ratio of the problem in the sense that
\[
    \mathrm{SNR} = \frac{\mathrm{sd}(\theta_f)}{\sigma}.
\]

In all our simulations, we have taken $n = 2048$, the errors to be IID $N(0,0.5)$ and $\mathrm{SNR} = 4$. The boxplots are based on $100$ Monte Carlo replications. We have used $5$-fold cross-validation (CV) to tune $\lambda$ for \laser{}. We have also used $5$-fold CV to tune the penalty parameter in TF. In each of Figures~\ref{fig:comp_blocks}, \ref{fig:comp_bumps}, \ref{fig:comp_heavisine} and \ref{fig:comp_doppler}, the left panel shows boxplots comparing the four methods and the right panel shows fits for one of these Monte Carlo realizations. 

In all but one of these experiments, \laser{} substantially outperforms the other three methods. For instance, we see that \laser{} captures more than six cycles (from the right) of the \texttt{Doppler} function accurately in the realization shown in Figure~\ref{fig:comp_doppler}. CSS also seems to do so, but it significantly overfits in the first cycle. TF, on the other hand, overfits much less in the first cycle but captures only about three cycles. Another noteworthy case is that of the \texttt{Bumps} function (see Figure~\ref{fig:comp_bumps}), where TF ($2$-nd order) does not appear to capture the interesting peaks. \laser{} (with degree $2$) does an excellent job in capturing most of these features while overfitting to a much lesser extent compared to both WT and CSS. For the \texttt{HeaviSine} function (see Figure~\ref{fig:comp_heavisine}), both \laser{} and CSS capture the discontinuity near $x = 0.7$, with \laser{} again overfitting to a lesser degree. (The other two methods both fail to capture this.) Finally, for the piecewise constant \texttt{Blocks} function, $0$-th order TF and \laser{} with $r = 0$ both significantly outperform the other two methods (see Figure~\ref{fig:comp_blocks}), with TF showing a slight edge over \laser{} in terms of RMSE.

Our numerical experiments suggest that the proposed method carries a lot of promise and can be a practically useful addition to the current nonparametric regression toolbox. The accompanying R package \texttt{laser} comes with a ready-to-use reference implementation of the dyadic version of \laser{}.

\section{Concluding remarks}\label{sec:Discussions}
In this section we discuss a few aspects and possible extensions of our 
estimator.

\noindent
\textbf{Beyond equispaced design.} Our estimator can be defined at points other than the design points $\{x_i\}_{i = 1}^{n}$ and, moreover, the design points need not be equispaced. For instance, at a point $x$ one can consider symmetric intervals around $x$ and sub-intervals just as before. The estimate will only be a function of the data points $y_i$ for which the corresponding design points $x_i$ fall within these intervals. 

\noindent
\textbf{Extensions to higher dimensions.} The ideas behind our estimator can be 
naturally generalised to higher dimensions. However, a rigorous proof of the 
corresponding risk bounds necessitates some new ideas and will appear in a 
forthcoming article.

\smallskip

\noindent \textbf{Acknowledgement.} SC's research was partly supported by the NSF Grant DMS-1916375. SG's research was partially supported by a grant from the 
Department of Atomic Energy, Government of India, under project 12R\&DTFR5.010500 
\sloppy and in part by a grant from the Infosys Foundation as a member of the 
Infosys-Chandrasekharan virtual center for Random Geometry. SSM was partially 
supported by an INSPIRE research grant (DST/INSPIRE/04/2018/002193) from the Department of Science and Technology, Government of India; a Start-Up Grant and the CPDA from the Indian Statistical Institute; and a Prime Minister Early Career Research Grant (ANRF/ECRG/2024/006704/PMS) from the Anusandhan National Research Foundation, Government of India. SC and SG are grateful to the Indian Statistical Institute, 
Kolkata for its hospitality during the early phases of the project. We 
thank Rajarshi Mukherjee and Adityanand Guntuboyina for many helpful discussions.

\bibliographystyle{plainnat}
\bibliography{refs}

\end{document}